\documentclass[lettersize,journal]{IEEEtran}
\usepackage{amsmath,amsfonts}
\usepackage{algorithmic}
\usepackage{algorithm}
\usepackage{array}
\usepackage[caption=false,font=normalsize,labelfont=sf,textfont=sf]{subfig}
\usepackage{textcomp}
\usepackage{stfloats}
\usepackage{url}
\usepackage{verbatim}
\usepackage{graphicx}
\usepackage{enumitem}
\usepackage{bbm}
\usepackage{algorithmic}
\usepackage{algorithm}

\usepackage{amsmath, amsthm, amssymb}
\usepackage{xcolor}
\newtheorem{proposition}{Proposition}

\newtheorem{theorem}{Theorem}
\usepackage{makecell} 
\usepackage[capitalise]{cleveref}
\usepackage{amsmath}
\usepackage{arydshln}
\usepackage{multirow}

\usepackage{amsthm}

\newtheorem{lemma}{Lemma}

\begin{document}

\title{A Unified Calibration Framework for Coordinate and Kinematic Parameters in Dual-Arm Robots}

\author{Tianyu Huang$^1$, Bohan Yang$^1$, Bin Li$^1$, Wenpan Li$^2$, Haoang Li$^3$, Wenlong Li$^2$, and Yun-Hui Liu$^1$
\thanks{$^{1}$Tianyu Huang, Bohan Yang, Bin Li, and Yun-Hui Liu are with the Hong Kong Centre for Logistics Robotics and Department of Mechanical and Automation Engineering, The Chinese University of Hong Kong, Hong Kong, China~(email: tianyuhuang@cuhk.edu.hk; bohanyang@cuhk.edu.hk; binli001@cuhk.edu.hk; yhliu@cuhk.edu.hk).}%
\thanks{$^{2}$Wenpan Li and Wenlong Li are with the State Key Laboratory of Intelligent Manufacturing Equipment and Technology, Huazhong University of Science and Technology, Wuhan, China~(email: wenpanli@hust.edu.cn; wlli@mail.hust.edu.cn). }
\thanks{$^{3}$Haoang Li is with the with the Thrust of Robotics and Autonomous Systems, The Hong Kong University of Science and Technology
(Guangzhou), Guangzhou, China~(email: haoangli@hkust-gz.edu.cn).}
}

\markboth{Journal of \LaTeX\ Class Files,~Vol.~14, No.~8, August~2021}%
{Shell \MakeLowercase{\textit{et al.}}: A Sample Article Using IEEEtran.cls for IEEE Journals}


\maketitle

\begin{abstract}
Precise collaboration in vision-based dual-arm robot systems requires accurate system calibration.
Recent dual-robot calibration methods have achieved strong performance by simultaneously solving multiple coordinate transformations. 
However, these methods either treat kinematic errors as implicit noise or handle them through separated error modeling, resulting in non-negligible accumulated errors.
In this paper, we present a novel framework for unified calibration of the coordinate transformations and kinematic parameters in both robot arms.
Our key idea is to unify all the tightly coupled parameters within a single Lie-algebraic formulation.
To this end, we construct \textit{a consolidated error model grounded in the product-of-exponentials formula}, which naturally integrates the coordinate and kinematic parameters in twist forms. 
Our model introduces no artificial error separation and thus greatly mitigates the error propagation.
In addition, we derive a closed-form analytical Jacobian from this model using Lie derivatives. By exploring the Jacobian rank property, we analyze the \textit{identifiability} of all calibration parameters and show that our joint optimization is well-posed under easily-satisfied conditions.
This enables the application of off-the-shelf iterative solvers to stably optimize these parameters on the manifold space.
Besides, to ensure robust convergence of our joint optimization, we develop a \textit{certifiably correct} algorithm for initializing the unknown coordinates.
Relying on semidefinite relaxation, our algorithm can yield a reliable initial coordinate estimate whose near-global optimality can be verified \textit{a posteriori}.
Extensive experiments validate the superior accuracy of our approach over previous baselines under identical visual measurements.
Meanwhile, our certifiable initialization consistently outperforms several coordinate-only baselines, proving its reliability as a starting point for joint optimization.
\end{abstract}

\section{Introduction}

The field of robotics is rapidly evolving towards greater autonomy and intelligence in recent years. 
In this context, vision-based dual-arm robot systems have emerged as a key platform for executing complex tasks that are beyond the capability of a single robot arm~\cite{Rakita-SR2019, Liu-ArXiv2024, Luo-SR2025}. 
These tasks, ranging from agile assembly of large-scale structures~\cite{Hartmann-TRO2022, Chen-TASE2021, Ge-RCIM2025} to collaborative manipulation of various daily objects~\cite{Suarez-SR2018, Colome-TRO2018, Wan-TII2019}, require not only high individual robot accuracy but also precise relative positioning between the two robots.
To realize such capabilities, a crucial problem lies in the calibration of dual-arm system parameters. 
Among these parameters, one primary category consists of the static coordinate transformations, such as the relative pose between the robot bases and the hand-eye transformation. These coordinate parameters, together with the robot kinematics, form the foundation of high-precision dual-arm cooperation tasks.
However, due to the strong coupling among these parameters and the inevitable robot motion errors, achieving accurate dual-arm calibration is quite challenging and continues to attract significant research interest~\cite{Wu-TRO2016, Qiao-PE2017, Jiang-TRO2023, Wang-TRO2021, Wang-TMECH2024}.


\begin{figure}[!t]
    \centering
    \includegraphics[width=0.36\textwidth]{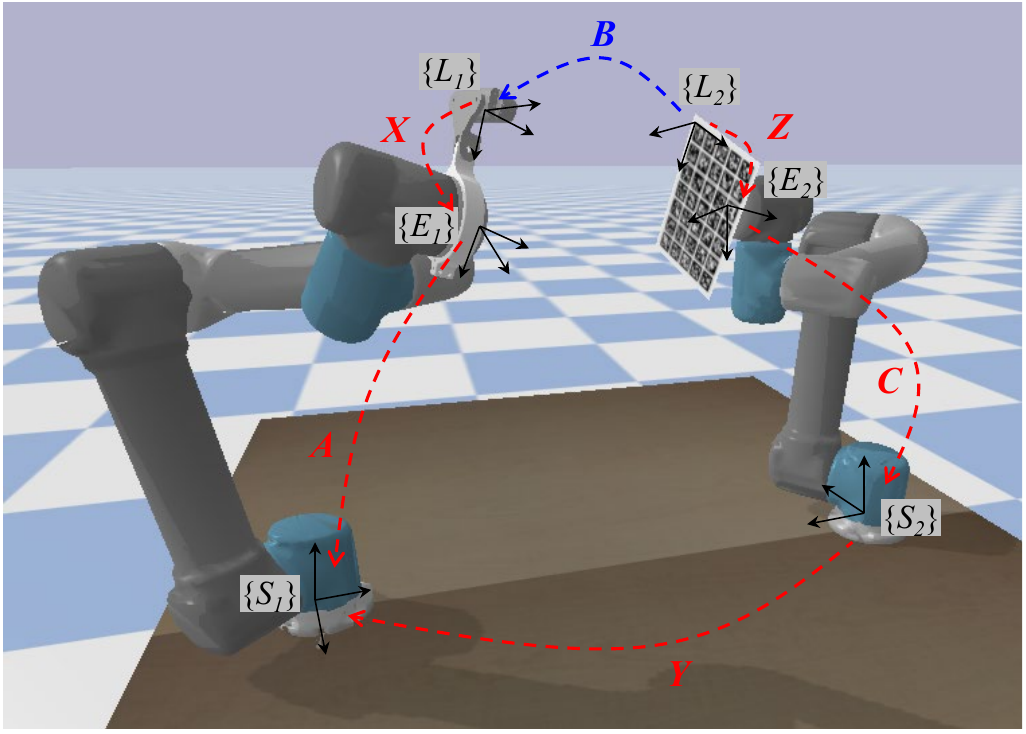}
    \\[-0.6em]
    \caption{Illustration of a canonical dual-arm robot system. The \textit{sensor-side} robot is equipped with a camera. The \textit{tool-side} robot holds a calibration tool~(e.g., a chessboard). The closed-loop pose chain in the dual-arm system can be defined by $\boldsymbol{A} \boldsymbol{X} \boldsymbol{B} = \boldsymbol{Y} \boldsymbol{C} \boldsymbol{Z}$.} 
    \label{fig:ill_dualrobot}
\end{figure}

Early attempts for dual‑arm calibration typically employ partitioned techniques and focus on parts of the coordinate transformations~\cite{Bonitz-RAM2002, Wang-TII2018, Zhu-IA2018, Ruan-ICRAE2017}. 
While these methods can be employed to calibrate the entire system in a step-by-step way, the accumulation of errors in each step often results in reduced calibration accuracy.
Recently, a series of simultaneous calibration methods are proposed to solve the above limitations~\cite{Wu-TRO2016, Jiang-TRO2023, Wang-TRO2021, Jiang-IA2021, Yan-Robotica2016, Zhu-TII2024}. 
The foundation of these works is the following pose chain in a canonical dual-robot system~(see Fig.~\ref{fig:ill_dualrobot}):
\begin{align} \label{eq:AXB=YCB}
    \boldsymbol{A} \boldsymbol{X} \boldsymbol{B} = \boldsymbol{Y} \boldsymbol{C} \boldsymbol{Z}, 
\end{align}
where $\boldsymbol{X}, \boldsymbol{Y}$, and $\boldsymbol{Z}$ are the unknown coordinate transformations (flange-to-sensor, base-to-base, and flange-to-tool); $\boldsymbol{A}$ and $\boldsymbol{C}$ define the end-effector poses computed from the robot kinematics; $\boldsymbol{B}$ denotes the tool pose with respect to the sensor and can be measured with high precision using a calibrated sensor-tool setup.
Relying on \eqref{eq:AXB=YCB} and a set of measurements $\{\mathbf{A}_i, \mathbf{C}_i, \mathbf{B}_i\}_{i=1}^m$ under multiple dual-arm posture configurations, these simultaneous methods aim to calibrate the three coordinate transformations at once, and achieves a remarkably higher accuracy compared to partitioned methods. This highlights the advantage of simultaneous calibration. 

However, most existing simultaneous calibration methods are \textit{coordinate-only}, that is, they do not explicitly model or optimize kinematic parameters, and instead treat kinematic-induced errors as the measurement noise in $\boldsymbol{A}$ and $\boldsymbol{C}$. This simplification can limit calibration accuracy in practical dual-arm systems, especially for industrial arms where the kinematic errors are non-negligible due to manufacturing tolerances and assembly errors.
Given this concern, a few recent studies attempt to incorporate kinematic calibration into dual-arm calibration. 
For example, the work in \cite{Wang-TMECH2024} shows that additionally compensating kinematic errors can improve the calibration accuracy; however, their method relies on external tracking devices and only considers the tool-side robot kinematics. 
A related method~\cite{Chen-TMECH2025} further incorporates the kinematics of both robots via an arm-wise decoupled error model for joint calibration. 
While effective, their kinematic deviations are aggregated at the residual level and corrected locally for each robot; this separated treatment can leave the coupling-induced errors insufficiently corrected, thereby causing error accumulation and suboptimal accuracy.

In this paper, we propose a novel dual-arm calibration framework that jointly calibrates the multiple coordinate transformations and kinematic parameters of both arms using a unified system-level error model. 
The key idea is to represent all the tightly coupled parameters within a single Lie-algebraic formulation, rather than treating the coupled error sources in a separated or arm-wise manner.
To this end, we reformulate the dual-robot calibration chain by leveraging the Lie group representation, resulting in \textit{a consolidated error model grounded in the product-of-exponentials~(PoE) formula}. 
This system-level PoE representation explicitly couples all calibration parameters within a single differentiable chain; as a result, it avoids artificial separation of the coupled error sources and reduces error accumulation along the loop. 
Based on the differentiable error model, we derive a closed-form analytical Jacobian via Lie derivatives. This Jacobian maps the increments of all parameters linearly to the geometric residual, transforming the intrinsically nonlinear joint calibration into a well-conditioned linear least-squares subproblem. 
We provide an explicit identifiability analysis by exploring the rank properties of the developed Jacobian, which is rarely explored in previous joint calibration works. The analysis shows that our joint optimization is well-posed under mild, easily satisfied excitation conditions. 
As a result, standard iterative solvers such as Gauss–Newton or Levenberg–Marquardt~\cite{Nocedal-Book2006} can be directly applied to optimize all calibration parameters on the SE(3) manifold space in a numerically stable manner.

Moreover, to ensure robust convergence of our joint optimization, we argue that a reliable initial estimate of the unknown coordinate transformations remains essential, particularly in the presence of substantial noise that can typically hinder iterative solvers. Although a natural choice is to apply existing coordinate-only calibration methods, we argue that a more robust and accurate coordinate initialization can be achieved.
To this end, we develop a certifiably correct solver tailored for the coordinate parameters. We rigorously reformulate the coordinate-only calibration problem as a Quadratically Constrained Quadratic Program (QCQP) and relax it into a Semidefinite Program (SDP). 
We then solve the convex SDP and obtain an initial coordinate estimate that can be \textit{a posteriori} verified to be near-globally optimal. 
Our SDP relaxation here avoids local optimization in many previous coordinate-only methods~\cite{Wu-TRO2016, Wang-TRO2021}, and yields a globally informed solution together with a global lower bound, making the initialization more reliable under kinematic-induced measurement noise.
As a result, the subsequent joint optimization is initialized within a favorable basin of attraction and can thus achieve higher estimation accuracy even in the presence of substantial kinematic errors.

Overall, the main contributions of this work can be summarized as follows:
\begin{itemize}
    \item We propose a unified dual-arm calibration framework that jointly optimizes the coordinate and kinematic parameters using a consolidated Lie-algebraic error model. Our error model avoids artificial separation of coupled error sources and thus largely mitigates error accumulation.

    \item We derive a closed-form analytical Jacobian from the proposed error model using Lie derivatives. By exploring the Jacobian rank properties, we provide an explicit identifiability analysis showing that our joint optimization is well-posed under mild excitation conditions.
    
    \item We develop a certifiably correct coordinate initialization method based on semidefinite relaxation, which can provide an initial coordinate estimate with an \textit{a posteriori} verification of near-global optimality.
\end{itemize}

Extensive simulations and real-world experiments are conducted to validate the superior accuracy of the proposed calibration framework compared to several baseline methods under identical measurements. Besides, our certifiable initialization constantly yields better initial coordinate estimates than several coordinate-only baselines, providing a more reliable starting point for our joint optimization.

The rest of this paper is organized as follows:
Section~\ref{sec:related} reviews the related vision-based robot calibration works. 
Section~\ref{sec:notion_preli} introduces the notations of this paper and provides some necessary preliminaries of Lie algebra and robot kinematics. 
Section~\ref{sec:overview} gives a short overview of our entire calibration algorithm. 
Section~\ref{sec:unified} introduces our unified optimization of the coordinate and kinematic parameters in dual-arm robot systems.
Section~\ref{sec:SDP_coord} introduces our certifiably correct solver for initializing the coordinate parameters.
Next, Section~\ref{sec:exper} presents the evaluation results of both simulated and real-world experiments.
Finally, we give a conclusion in Section~\ref{sec:conclu}.

\section{Related Works}
\label{sec:related}

Robot calibration is a long-standing prerequisite for precise robotic manipulation tasks. It aims to estimate the unknown transformation matrices between robots and external coordinate systems, as well as the geometric parameters within the robot kinematic structure, so as to improve the absolute positioning accuracy~\cite{Corke-BOOK2011, Siciliano-BOOK2009}.

In vision-based settings, a representative line of research is the \textbf{hand–eye and robot–world calibration}. 
Depending on the estimation objective, the underlying pose-chain formulations of these works can be categorized into two canonical forms: the hand-eye calibration equation $\boldsymbol{A}\boldsymbol{X} = \boldsymbol{X}\boldsymbol{B}$ and the robot-world/hand-eye calibration equation $\boldsymbol{A}\boldsymbol{X} = \boldsymbol{Y}\boldsymbol{B}$. 
Early works on $\boldsymbol{A}\boldsymbol{X} = \boldsymbol{X}\boldsymbol{B}$ typically focus on deriving linearized formulations with efficient closed-form solutions~\cite{Shiu-TRA1989, Tsai-TRA1989, Park-TRA1994, Daniilidis-IJRR1999}. Later, to improve the accuracy against measurement noise, many researchers employ the iterative local refinement via nonlinear least-squares optimization~\cite{Heller-ICRA2014, Horaud-IJRR1995, Koide-RAL2019}, and more recently, begin to explore the global optimization methods~\cite{Zhao-ICRA2011, Heller-TPAMI2015}.
As for $\boldsymbol{A}\boldsymbol{X} = \boldsymbol{Y}\boldsymbol{B}$, the robot–world/hand–eye calibration introduces an additional unknown transformation from the robot base to the world coordinate system. Beyond the classical linear solutions~\cite{Zhuang-TRA1994, Li-IJPS2010, Shah-JMR2013}, increasing efforts also begin to be devoted to iterative nonlinear optimization methods~\cite{Dornaika-2002TRA, Strobl-IROS2006, Zhao-PRL2019, Wise-ArXiv2025} that can jointly refine the two unknown transformations, in order to better accommodate measurement noise and reduce the coupling error.

Despite the above established progress in single-robot systems, achieving high-precision calibration in dual-robot systems cannot be obtained by a straightforward extension.
As a dual-robot setup typically involves multiple coupled coordinate transformations~(see Fig.~\ref{fig:ill_dualrobot}); besides, the kinematic errors of both robots can largely magnify the systematic errors in the pose chain. 
Early studies, therefore, often adopt \textbf{partitioned strategies}, leveraging single-robot calibration tools while focusing on estimating only a subset of the coordinate transformations~\cite{Bonitz-RAM2002, Wang-TII2018, Zhu-IA2018, Ruan-ICRAE2017}. 
In particular, \cite{Bonitz-RAM2002} and \cite{Ha-IA2023} incorporate additional physical/fixture constraints to estimate the relative pose between the two robot bases, whereas \cite{Ruan-ICRAE2017} and \cite{Zhao-IJAMT2018} obtain this transformation with the aid of an external measuring device~(e.g., a 3D scanner or laser tracker) to provide accurate spatial references.
Moreover, \cite{Zhu-IA2018} and its extension~\cite{Zhu-BPAS2019} utilize the virtual constraint of the camera optical axis to estimate the tool--flange transformation and further refine the kinematic parameters of a dual-manipulator system. Besides, \cite{Fu-RAL2020} involves a dual-quaternion-based pipeline that sequentially solves the hand-eye transformations and the other two transformations~(base-to-base and flange-to-tool).
While such partitioned methods can be employed to calibrate the whole dual-robot system in a step-by-step way, they typically require multiple calibration phases, and thus, the accuracy is always limited by the accumulated calibration errors.

To reduce calibration effort and mitigate error accumulation, many works study \textbf{simultaneous estimation of multiple coordinate transforms} directly from the coupled dual-robot pose chain in \eqref{eq:AXB=YCB}, without decomposing it into separate stages~\cite{Wu-TRO2016, Yan-Robotica2016, Ma-AR2018, Wang-TRO2021, Qin-Sensors2022}. In particular, \cite{Wang-IROS2014} and its extension~\cite{Wu-TRO2016} involve a closed-form initialization for the coupled rotational components based on quaternion algebra, and then perform an iterative refinement via first-order linearization on the Lie group. Despite their pioneering status, the translational components are recovered only in closed form (i.e., without being jointly refined with rotations), which can lead to sub-optimal results due to rotation–translation coupling.
Later, the work in \cite{Wang-TRO2021} introduces a closed-form initialization based on the Kronecker product and, more importantly, an iterative scheme that jointly optimizes both rotational and translational parts of all unknown coordinate transforms within a unified objective. The experimental results demonstrate the superior accuracy of their joint optimization of rotational and translational parts against the rotation-only optimization in~\cite{Wu-TRO2016}.
Recently, the authors in \cite{Jiang-TRO2023} propose to cast the simultaneous calibration into a quadratic-form optimization model: by leveraging the properties of Lie algebra representation on SE(3), they employ convex optimization together with Newton-like iterations to simultaneously solve for the coupled rotations and translations.
Despite the enhanced accuracy of these simultaneous methods, the majority of them remain \emph{coordinate-only}: they implicitly assume accurate robot kinematics and treat kinematic-induced biases in the robot-reported poses as measurement noises, which inevitably limit their achievable accuracy when the kinematic errors of the robots are prominent and non-negligible.

In fact, \textbf{robot kinematic calibration} is also crucial for improving system accuracy and has been investigated extensively, leading to many mature techniques. Existing kinematic calibration approaches can be roughly categorized by the underlying modeling formulation, e.g., the standard Denavit-Hartenberg (D-H) model~\cite{Denavit-ASME1955} and its variants~\cite{Everett-ICRA1987, Hayati-ICRA1988}, the complete and parametrically continuous~(CPC) models~\cite{Zhuang-ICRA1990}, and the product-of-exponentials (PoE) representation~\cite{Park-Book1994, He-TRO2010}. The representative PoE-based modeling is particularly attractive for its effectiveness in systematic perturbation-based error modeling~\cite{Wu-TASE2014, Qiao-PE2017, Chen-TRO2026}.
However, integrating these kinematic calibration ideas into dual-arm robot calibration remains nontrivial: the tight coupling between the multiple coordinate transforms and the kinematic deviations of both arms substantially complicates the modeling and optimization.
Currently, only a few recent works begin to explore \textbf{joint coordinate-kinematic calibration} for dual-arm robot systems.
In~\cite{Wang-TMECH2024}, the authors use an external tracking device to construct a kinematic error transfer model for joint refinement, optimizing multiple coordinate transforms together with the kinematic parameters of the tool-side arm. Their results demonstrate that explicit kinematic compensation can noticeably improve calibration accuracy, but the approach relies on external metrology and only refines the tool-side kinematics.
In contrast, \cite{Chen-TMECH2025} removes external tracking and incorporates the kinematics of both robots via an arm-wise decoupled error model; while effective, their formulation requires a hierarchical, multi-stage optimization to handle the high-dimensional coupling, and the residual coupling effects can accumulate along the dual-arm pose chain, especially under large kinematic errors.

Given the above research background, we propose in this paper a unified calibration framework that jointly estimates and refines the multiple coordinate transformations together with the kinematic parameters of both robot arms. Our target is to reduce the kinematic-induced bias and mitigate the error accumulation through the dual-arm chain, so as to achieve system calibration with higher precision.

\section{Notations and Preliminaries}
\label{sec:notion_preli}

\begin{figure*}[!t]
    \centering
    \includegraphics[width=0.9\textwidth]{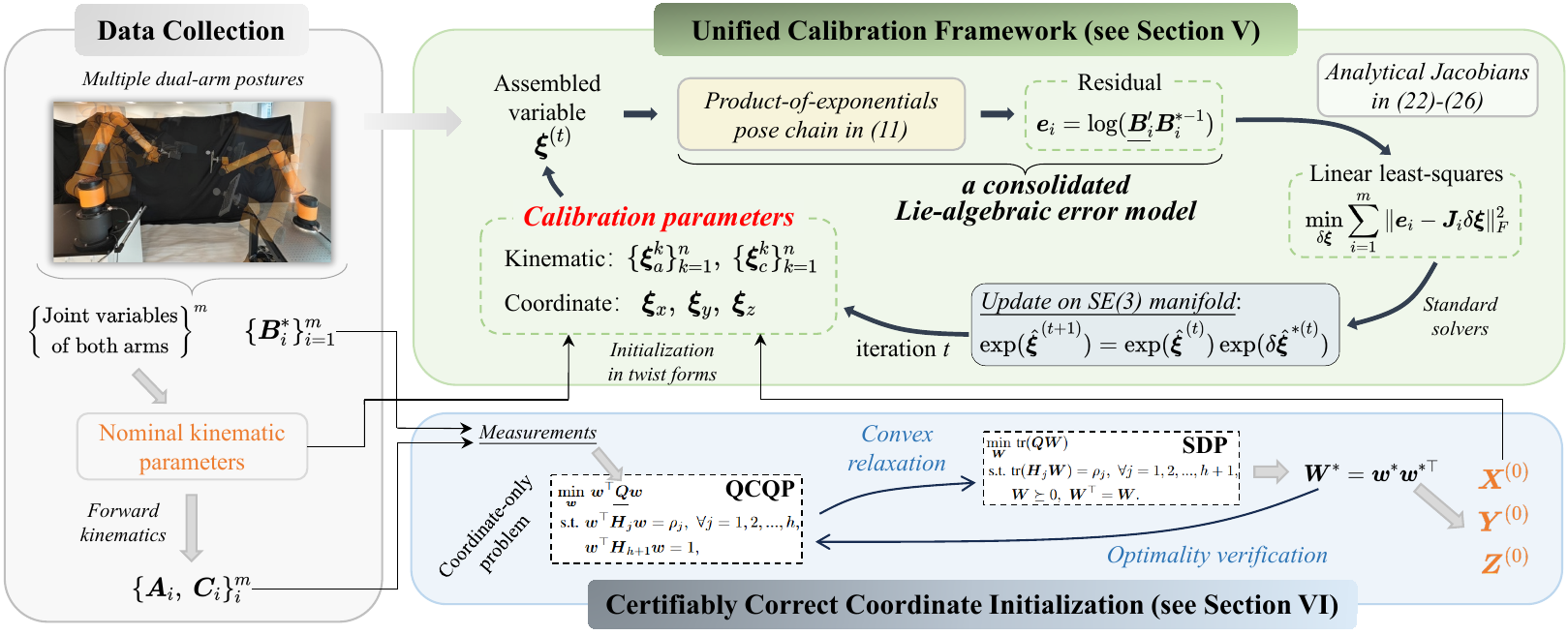}
    \vspace{-0.5em}
    \caption{Pipeline of the proposed unified calibration algorithm. Given a set of measurements under multiple dual-arm postures, we construct a consolidated Lie-algebraic error model and iteratively refine the coordinate and kinematic parameters on the $\mathrm{SE}(3)$ manifold. Before our iterative optimization, the coordinate transforms are initialized by a certifiably correct SDP relaxation solver, while the nominal kinematic parameters directly serve as the initial guess.}
    \label{fig:pipeline}
\end{figure*}

\textbf{Scalars, Vectors, Matrices}.
We employ lowercase characters~(e.g.,~$v$) to represent real scalars, lowercase boldface characters~(e.g.,~$\boldsymbol{v}$) for real column vectors, and uppercase boldface characters~(e.g.,~$\boldsymbol{V}$) for real matrices. 
$\mathbf{0}_d$ and $\mathbf{0}_{d \times d}$ represent the column vectors and matrix of all zeros with size $d$ and $d\times d$, respectively.
$\mathbf{I}_d$ denotes the identity matrix of size $d \times d$.
The $L_2$-norm of a vector is denoted as $\|\cdot\|_2$. The Frobenius norm of a matrix is denoted as $\|\cdot\|_{F}$.
For a vector $\boldsymbol{v}$, we use $\boldsymbol{v}_{(i)}$ to represent its $(i)$ entry.
For a matrix $\boldsymbol{V}$, we employ $\boldsymbol{V}_{(i, j)}$ to represent its $(i, j)$ entry.
$\mathrm{diag}(\boldsymbol{V}_1,\boldsymbol{V}_2)$ represents a block-diagonal matrix with $\boldsymbol{V}_1$ and $\boldsymbol{V}_2$ placed on the diagonal and zeros elsewhere.
Besides, we define $\textnormal{vec}(\cdot)$ as a vectorization operator that reorders a matrix by stacking its columns. 
The symbols $\textnormal{rank}(\cdot)$ and $\textnormal{tr}(\cdot)$ denote the rank and trace of the provided matrix, respectively.
In addition, the standard Kronecker product of two matrices $\boldsymbol{V}_1$ and $\boldsymbol{V}_2$ is defined by $\boldsymbol{V}_1\otimes\boldsymbol{V}_2$. 

\textbf{Sets and Frames}.
We use $\mathcal{S}^n$ to denote the sets of real symmetric matrices with size $n\times n$. Furthermore, $\mathcal{S}_{+}^{n} \doteq \{\boldsymbol{M} \in \mathcal{S}^n : \boldsymbol{M} \succeq 0 \}$ denotes the set of symmetric \textit{positive semidefinite}~(PSD) matrices with size $n\times n$.
When the dimension is clear from the context, we may simply write $\mathcal{S}$ and $\mathcal{S}_{+}$.
Besides, $\textnormal{SO}(3)$ denotes the $3$-D special orthogonal group and $\textnormal{SE}(3)$ denotes the special $3$-D special Euclidean group.
Each right-hand reference frame in this paper is denoted as $\{F\}$. The rigid body motion from frame $\{S\}$ to frame $\{E\}$ can be written as a rigid transformation matrix $^E\boldsymbol{T}_S = \begin{bmatrix}
    ^E\boldsymbol{R}_S & ^E\boldsymbol{t}_S \\
    \boldsymbol{0}_3^{\top} & 1
\end{bmatrix} \in \textnormal{SE}(3) \subset \mathbb{R}^{4\times4}$, where $^E\boldsymbol{R}_S \in \textnormal{SO}(3)$ denotes the rotation and $^E\boldsymbol{t}_S \in \mathbb{R}^3$ denotes the translation.
In the dual-robot calibration problem~(as illustrated in Fig.~\ref{fig:ill_dualrobot}), there are
\begin{equation}
    \begin{aligned}
        \boldsymbol{A} &\doteq {^{S_1}\boldsymbol{T}_{E_1}}, & 
        \boldsymbol{X} &\doteq {^{E_1}\boldsymbol{T}_{L_1}}, & 
        \boldsymbol{B} &\doteq {^{L_1}\boldsymbol{T}_{L_2}}, \\
        \boldsymbol{Y} &\doteq {^{S_1}\boldsymbol{T}_{S_2}}, & 
        \boldsymbol{C} &\doteq {^{S_2}\boldsymbol{T}_{E_2}}, & 
        \boldsymbol{Z} &\doteq {^{E_2}\boldsymbol{T}_{L_2}},
    \end{aligned}
\end{equation}
where $\{S_1\}$, $\{E_1\}$, and $\{L_1\}$ denote the basic coordinate system, the end-flange coordinate system, and the sensor coordinate system in the sensor-side robot, respectively; $\{S_2\}$, $\{E_2\}$, and $\{L_2\}$ denote the basic coordinate system, the end-flange coordinate system, and the tool coordinate system in the tool-side robot, respectively.

\textbf{Lie Algebra}.
SE(3) and SO(3) are two representative examples of \textit{Lie group}. Associated with any matrix Lie group is its \textit{Lie Algebra} — here, $se$(3) and $so$(3) — which offers a powerful linear representation for describing small changes in pose.
The element of $se$(3) is called a twist, denoted as a matrix by 
\begin{align}
    \hat{\boldsymbol{\xi}} = \begin{bmatrix}
    \boldsymbol{w}_\times & \boldsymbol{\rho} \\
    \boldsymbol{0}_3^{\top} & 0
    \end{bmatrix} \in se(3),
\end{align}
where $\boldsymbol{w} \in \mathbb{R}^3$ encodes the rotation, $\boldsymbol{\rho} \in \mathbb{R}^3$ encodes the translation, and $\boldsymbol{w}_\times$ defines the skew-symmetric matrix of $\boldsymbol{w}$. 
For simplicity, we can use a 6-dimensional vector $\boldsymbol{\xi} = [\boldsymbol{w}^\top, \boldsymbol{\rho}^\top]^\top \in \mathbb{R}^6$ to represent the twist coordinate of $\hat{\boldsymbol{\xi}}$. Let us define $\vee$ as an operator that maps $\hat{\boldsymbol{\xi}}$ into $\boldsymbol{\xi}$, i.e., $[\hat{\boldsymbol{\xi}}]^{\vee} = \boldsymbol{\xi}$. 
The connection between a $\boldsymbol{T} \in $ SE(3) and its related twist $\hat{\boldsymbol{\xi}} \in se(3)$ is established through exponential map, i.e.,
\begin{align}
    \boldsymbol{T} = \exp (\hat{\boldsymbol{\xi}}) \in \text{SE}(3).
\end{align}
Conversely, we have $\hat{\boldsymbol{\xi}} = \log(\boldsymbol{T}) \in se(3)$ through logarithmic map. 
Regarding the inverse operation, there is $\boldsymbol{T}^{-1} = \exp(-\hat{\boldsymbol{\xi}})$.
As to the transformation propagation, the adjoint transformation is employed. 
For a twist $\boldsymbol{\xi}$ expressed in one frame, the equivalent twist in another frame transformed by $\dot{\boldsymbol{T}}\in $ SE(3) is given by $ \text{Ad}(\dot{\boldsymbol{T}})\boldsymbol{\xi}$ and 
\begin{align}
    \text{Ad}(\dot{\boldsymbol{T}}) = \begin{bmatrix}
        \dot{\boldsymbol{R}} & \mathbf{0}_{3 \times 3} \\
        \dot{\boldsymbol{t}}_\times\dot{\boldsymbol{R}} & \dot{\boldsymbol{R}}
    \end{bmatrix} \in \mathbb{R}^{6\times6},
\end{align}
where $\dot{\boldsymbol{R}} \in $ SO(3) and $\dot{\boldsymbol{t}} \in \mathbb{R}^3$ are the corresponding rotation and translation in $\dot{\boldsymbol{T}}$.

\textbf{Robot Kinematics}.
For a standard $n$-degree-of-freedom serial robot arm, the end-effector pose is often computed using the following forward kinematics:
\begin{align}
    ^{S}\boldsymbol{T}_{E} = ^{0}\boldsymbol{T}_{n} = ^{0}\boldsymbol{T}_{1}\ ^{1}\boldsymbol{T}_{2}\ \cdots {}^{n-1}\boldsymbol{T}_{n},
\end{align}
where each $^{k}\boldsymbol{T}_{k+1}$~($0\leq k\leq n-1$) denotes the homogeneous transformation of the $(k+1)$th joint frame relative to the $k$th joint frame.
This frame-to-frame composition can be viewed as a sequence of rigid-body motions on SE(3) induced by the joint variables.
Although $^{k}\boldsymbol{T}_{k+1}$ is commonly parameterized via local frame conventions (e.g., D-H modeling~\cite{Denavit-ASME1955}), the entire kinematic chain admits an equivalent system-level representation expressed in the form of Lie algebra.
Namely, the forward kinematics can be compactly written in the product-of-exponentials~(PoE) form~\cite{Park-Book1994}:
\begin{align}
^{S}\boldsymbol{T}_{E} =
\exp(\hat{\boldsymbol{\xi}}^1q^1)\exp(\hat{\boldsymbol{\xi}}^2q^2)\cdots\exp(\hat{\boldsymbol{\xi}}^nq^n)\exp(\hat{\boldsymbol{\xi}}^{st}),
\end{align}
where $\hat{\boldsymbol{\xi}}^1,\ldots,\hat{\boldsymbol{\xi}}^n \in se(3)$ are the joint twists expressed in the base frame,
$\hat{\boldsymbol{\xi}}^{st}$ encodes a user-defined zero reference configuration,
and $q^1,\ldots,q^n \in [-\pi,\pi]$ denote the joint variables.
$\{\hat{\boldsymbol{\xi}}^1,\ldots,\hat{\boldsymbol{\xi}}^n,\hat{\boldsymbol{\xi}}^{st}\}$ constitutes the kinematic parameters of the robot and remains invariant across configurations, while the joint variables vary with robot motion.




\section{Algorithm Overview}
\label{sec:overview}

Fig.~\ref{fig:pipeline} gives an overview of our unified calibration algorithm. To begin with, we collect $m$ sets of measurement samples under multiple distinct dual-arm posture configurations.
Each sample consists of the measured posture of the calibration tool in the sensor frame, i.e., $\boldsymbol{B}_i^*$, together with the joint variables of both robot arms.
Given these measurements, we aim to jointly calibrate the three significant coordinate transformations and the kinematic parameters of both robot arms.

As the calibration parameters are high-dimensional and tightly coupled, a reliable initialization is important for fast and stable convergence.
Regarding the kinematics, we directly use the nominal parameters provided by the robot controllers as the initial guess.
As for the unknown coordinate transforms, we design a \textit{certifiably correct solver}~(see Section~\ref{sec:SDP_coord}). 
We first reformulate the coordinate-only calibration into a QCQP problem and then solve its convex SDP relaxation.
Our SDP relaxation provides an \textit{a posteriori} certificate through its global lower bound, allowing us to verify the near-global optimality of the recovered initialization and thus place the subsequent joint optimization in a favorable basin of attraction.

Built upon the above initialization, we perform the core \emph{unified optimization}~(see Section~\ref{sec:unified}).
Specifically, we reformulate the dual-arm pose chain into a consolidated Lie-algebraic error model grounded in the product-of-exponentials (PoE) formula, so that all coordinate and kinematic parameters are expressed in twist forms within a single system-level error formula.
Based on this error model, we derive a closed-form analytical Jacobian via Lie derivatives, which maps the increments of all parameters linearly to the geometric residual.
This derivation transforms the intrinsically nonlinear joint calibration into a well-conditioned linear least-squares subproblem, enabling standard iterative solvers to simultaneously update all parameters at each iteration.
Each update is applied through exponential mapping to remain on the SE(3) manifold space, and the iterations converge when the magnitude of the increment vector becomes sufficiently small.
Finally, our algorithm yields a set of high-precision coordinate transformations and kinematic parameters that directly support high-precision dual-arm collaboration tasks.

\section{Unified Calibration of Coordinate and Kinematic Parameters}
\label{sec:unified}
As introduced in Section~\ref{sec:overview}, our overall pipeline contains a certifiable coordinate initialization and a unified calibration.
The unified optimization is the core of the proposed framework: it explicitly models the strong coupling between the coordinate transforms and the kinematic parameters of both arms, and jointly updates them within a single system-level objective, thereby mitigating the error accumulation inherent to decoupled or arm-wise calibration.
This section establishes (i) the consolidated Lie-algebraic error model based on the dual-arm pose chain; (ii) an efficient iterative solver equipped with a closed-form analytical Jacobian derived from Lie derivatives; (iii) an identifiability analysis on the derived Jacobian.

\subsection{Error Model via Product-of-Exponentials Formula}

We begin by reformulating the standard dual-robot pose chain $\boldsymbol{A}_i\boldsymbol{X}\boldsymbol{B}_i = \boldsymbol{Y}\boldsymbol{C}_i\boldsymbol{Z}$ to the following formula that isolates the measurable transformation $\boldsymbol{B}_i$:
\begin{align} \label{eq:B_chain}
    \boldsymbol{B}_i = \boldsymbol{X}^{-1} \boldsymbol{A}_i^{-1} \boldsymbol{Y} \boldsymbol{C}_i \boldsymbol{Z}.
\end{align}
Here, in previous coordinate-only methods~\cite{Wang-TRO2021, Jiang-TRO2023}, the robot end-postures $\boldsymbol{A}_i$ and $\boldsymbol{C}_i$ are typically obtained from the robot controllers and thus contain kinematic errors. Now we assume that $\boldsymbol{X}$, $\boldsymbol{Y}$, and $\boldsymbol{Z}$ are initialized via a coordinate calibration algorithm from the noisy kinematics. 
Our unified calibration proceeds by first reparameterizing all transformations in~\eqref{eq:B_chain} using a Lie-algebraic formulation.

Consider the end postures of the two robots, i.e., $\boldsymbol{A}_i$ and $\boldsymbol{C}_i$.  
Applying the Lie-algebraic kinematic representation to both robots~(see Section~\ref{sec:notion_preli}), we denote the sensor‑side robot~(ref. $\boldsymbol{A}_i$) parameters as $\{\hat{\boldsymbol{\xi}}_a^1, ..., \hat{\boldsymbol{\xi}}_a^n, \ \hat{\boldsymbol{\xi}}^{st}_a,\ q^1_{a_i}, ..., q^n_{a_i}\}$, and the tool-side robot~(ref. $\boldsymbol{C}_i$) parameters as $\{\hat{\boldsymbol{\xi}}_c^1, ..., \hat{\boldsymbol{\xi}}_c^n,\ \hat{\boldsymbol{\xi}}^{st}_c,\ q^1_{c_i}, ..., q^n_{c_i}\}$.
Then the two end-postures can be expressed as:
\begin{align}
    \boldsymbol{A}_i^{-1} &= \exp(-\hat{\boldsymbol{\xi}}^{st}_a)\exp(-\hat{\boldsymbol{\xi}}^n_aq^n_{a_i})\cdots\exp(-\hat{\boldsymbol{\xi}}^1_aq^1_{a_i}), \\
    \boldsymbol{C}_i & = \exp(\hat{\boldsymbol{\xi}}^1_cq^1_{c_i})\cdots\exp(\hat{\boldsymbol{\xi}}^n_cq^n_{c_i})\exp(\hat{\boldsymbol{\xi}}^{st}_c). \
\end{align}
Furthermore, we let $\hat{\boldsymbol{\xi}}_x$, $\hat{\boldsymbol{\xi}}_y$, and $\hat{\boldsymbol{\xi}}_z$ define the twists associated with $\boldsymbol{X}$, $\boldsymbol{Y}$, and $\boldsymbol{Z}$, respectively. Substituting the above expressions into \eqref{eq:B_chain} gives:
\begin{align} \label{eq:B_PoE}
    \boldsymbol{B}_i = & \exp(-\hat{\boldsymbol{\xi}}_x) \exp(-\hat{\boldsymbol{\xi}}^{st}_a)\exp(-\hat{\boldsymbol{\xi}}^n_aq^n_{a_i})\cdots\exp(-\hat{\boldsymbol{\xi}}^1_aq^1_{a_i}) \notag \\ 
    & \times \exp(\hat{\boldsymbol{\xi}}_y)\exp(\hat{\boldsymbol{\xi}}^1_cq^1_{c_i})\cdots\exp(\hat{\boldsymbol{\xi}}^n_cq^n_{c_i}) \notag \\
    & \times \exp(\hat{\boldsymbol{\xi}}^{st}_c) \exp(\hat{\boldsymbol{\xi}}_z),
\end{align}
where $\times$ denotes the matrix multiplication here for clarity. 
The PoE formula in \eqref{eq:B_PoE} yields a compact and fully differentiable system-level model that places all coordinate and kinematic quantities in a single Lie-algebraic parameterization. Accordingly, we are enabled to construct a unified error model.

Before proceeding, we clarify a key design choice: \textit{which parameters in~\eqref{eq:B_PoE} should be optimized in the unified calibration}. This problem is essential in ensuring a minimal and well-posed estimation. A straightforward choice is to jointly estimate all twists appearing in \eqref{eq:B_PoE}. However, we argue that $\hat{\boldsymbol{\xi}}^{st}_a$ and $\hat{\boldsymbol{\xi}}^{st}_c$ can be directly fixed without additional optimization. There are two reasons: first, they encode user-defined PoE reference (``zero'') frames and optimizing them introduce redundant convention-dependent degrees of freedoms; second, the products $\exp(-\hat{\boldsymbol{\xi}}_x)\exp(-\hat{\boldsymbol{\xi}}^{st}_a)$ and $\exp(\hat{\boldsymbol{\xi}}^{st}_c)\exp(\hat{\boldsymbol{\xi}}_z)$ admit equivalent reparameterization, such that variations in $\hat{\boldsymbol{\xi}}^{st}_a$ (resp. $\hat{\boldsymbol{\xi}}^{st}_c$) can be absorbed into $\hat{\boldsymbol{\xi}}_x$ (resp. $\hat{\boldsymbol{\xi}}_z$) without affecting the predicted $\boldsymbol{B}_i$. 
Given the above properties, we keep $\hat{\boldsymbol{\xi}}^{st}_a$ and $\hat{\boldsymbol{\xi}}^{st}_c$ at their nominal (user-defined) values during our calibration and only optimize 
\begin{align}
    \boldsymbol{\xi}= [ \boldsymbol{\xi}_x^\top,\boldsymbol{\xi}_y^\top,\boldsymbol{\xi}_z^\top,\ \boldsymbol{\xi}^{1\top}_a,\ldots, \boldsymbol{\xi}^{n\top}_a,\ \boldsymbol{\xi}^{1\top}_c,\ldots, \boldsymbol{\xi}^{n\top}_c ].
\end{align}

Now we return to the derivation of our error model. We let $\boldsymbol{B}_i^*$ denote the high-precision measured transformation and $\boldsymbol{B}_i'$ denote the transformation computed by the pose chain in \eqref{eq:B_PoE}. 
Ideally, $\boldsymbol{B}_i'$ should coincide with $\boldsymbol{B}_i^*$; however, inaccuracies in the coordinates and kinematics cause a deviation between them. 
We define the calibration error $\boldsymbol{e}_i \in se(3)$ as
\begin{align} \label{eq:error_1}
    \boldsymbol{e}_i & \doteq \log(\boldsymbol{B}_i'\boldsymbol{B}_i^{*-1}) = \sum_{k=1}^\infty(-1)^{k-1} \frac{(\boldsymbol{B}_i'\boldsymbol{B}_i^{*-1}-\mathbf{I}_{4})^k}{k}.
\end{align}
Since the incremental errors during iterative optimization are typically small, we apply a first-order approximation following~\cite{He-TRO2010, Park-Book1994} and obtain
\begin{align} \label{eq:error_2}
    \boldsymbol{e}_i \approx \boldsymbol{B}_i'\boldsymbol{B}_i^{*-1} - \mathbf{I}_{4} = (\boldsymbol{B}_i' - \boldsymbol{B}_i^*)\boldsymbol{B}_i^{*-1} = \delta\boldsymbol{B}_i \boldsymbol{B}_i^{-1}.
\end{align}
The deviation term $\delta\boldsymbol{B}_i\, \boldsymbol{B}_i^{*-1}$ is obtained by differentiating~\eqref{eq:B_PoE} with respect to the parameters in $\boldsymbol{\xi}$. Denoting the corresponding increments by $\delta\boldsymbol{\xi}_x,\delta\boldsymbol{\xi}_y,\delta\boldsymbol{\xi}_z,\delta\boldsymbol{\xi}_a,\delta\boldsymbol{\xi}_c$, we have
\begin{align}\label{eq:error_derivation}
    \delta\boldsymbol{B}_i \boldsymbol{B}_i^{-1} = \Bigg( \substack{\frac{\partial \boldsymbol{B}_i}{\partial \boldsymbol{\xi}_x}{\partial \boldsymbol{\xi}_x} + \frac{\partial \boldsymbol{B}_i}{\partial \boldsymbol{\xi}_y}{\partial \boldsymbol{\xi}_y} + \frac{\partial \boldsymbol{B}_i}{\partial \boldsymbol{\xi}_z}{\partial \boldsymbol{\xi}_z} \\ + \frac{\partial \boldsymbol{B}_i}{\partial \boldsymbol{\xi}_a}{\partial \boldsymbol{\xi}_a} + \frac{\partial \boldsymbol{B}_i}{\partial \boldsymbol{\xi}_c}{\partial \boldsymbol{\xi}_c}} \Bigg) \boldsymbol{B}_i^{-1}.
\end{align}
where we define $\boldsymbol{\xi}_a = [\boldsymbol{\xi}_a^{1\top}, ..., \boldsymbol{\xi}_a^{n\top}]^\top \in \mathbb{R}^{6n}$ and $\boldsymbol{\xi}_c = [\boldsymbol{\xi}_c^{1\top}, ..., \boldsymbol{\xi}_c^{n\top}]^\top \in \mathbb{R}^{6n}$ here for notational simplicity. 

Now, given $m$ sets of measurements, we can construct a least‑squares minimization problem to identify the coordinate and kinematic parameters. The cost function is formulated by summing the residuals between the calibration error and the deviation term, i.e., 
\begin{align}\label{eq:prob_quadratic}
    \min_{\substack{{\partial \boldsymbol{\xi}_x}, {\partial \boldsymbol{\xi}_y}, {\partial \boldsymbol{\xi}_z},\\ {\partial \boldsymbol{\xi}_a}, {\partial \boldsymbol{\xi}_c}}} \sum_{i=1}^m \|  \boldsymbol{e}_i - \Bigg( \substack{\frac{\partial \boldsymbol{B}_i}{\partial \boldsymbol{\xi}_x}{\partial \boldsymbol{\xi}_x} + \frac{\partial \boldsymbol{B}_i}{\partial \boldsymbol{\xi}_y}{\partial \boldsymbol{\xi}_y} + \frac{\partial \boldsymbol{B}_i}{\partial \boldsymbol{\xi}_z}{\partial \boldsymbol{\xi}_z} \\ + \frac{\partial \boldsymbol{B}_i}{\partial \boldsymbol{\xi}_a}{\partial \boldsymbol{\xi}_a} + \frac{\partial \boldsymbol{B}_i}{\partial \boldsymbol{\xi}_c}{\partial \boldsymbol{\xi}_c}} \Bigg) \boldsymbol{B}_i^{-1} \|_F^2 .
\end{align}
The unified objective in \eqref{eq:prob_quadratic} aims to simultaneously correct all coordinate transformations and the kinematic parameters of both robots within a single cost function. This unified optimization inherently avoids the error accumulation that arises from sequential calibration and separated error modeling. The following section will detail how we derive an efficient iterative solver by exploiting the analytical structure of the Jacobian associated with this PoE‑based error model.

\subsection{Analytical Jacobian via Lie Derivatives}
\label{subsec:Jaco}
To enable efficient and robust optimization of \eqref{eq:prob_quadratic}, an explicit and exact Jacobian matrix is essential. 
This section details the derivation of this Jacobian in closed form by leveraging the differential calculus of Lie groups, i.e., Lie derivatives. 
We show that the complex identification can be transformed into a well-conditioned linear problem, enabling precise recovery of all coordinate and kinematic parameters.

The foundational step is to express the first-order approximation of the calibration error, i.e., $\delta\boldsymbol{B}_i \boldsymbol{B}_i^{-1}$, as a linear function of all the to-be-calibrated parameters, i.e.,
\begin{align}\label{eq:devia_Jaco}
[\delta\boldsymbol{B}_i \boldsymbol{B}_i^{-1}]^{\vee} = \boldsymbol{J}_i \delta\boldsymbol{\xi},
\end{align}
where
\begin{align}
    \boldsymbol{J}_i &= [\boldsymbol{J}_x^{i}, \boldsymbol{J}_y^{i}, \boldsymbol{J}_z^{i}, \boldsymbol{J}_{\xi_a^1}^{i}, ..., \boldsymbol{J}_{\xi_a^n}^{i}, \notag \\ & \quad \ \ \boldsymbol{J}_{\xi_c^1}^{i}, ..., \boldsymbol{J}_{\xi_c^n}^{i}] \in \mathbb{R}^{6\times(12n+18)}, \\
    \delta\boldsymbol{\xi} &= [\delta\boldsymbol{\xi}_x^\top, \delta\boldsymbol{\xi}_y^\top, \delta\boldsymbol{\xi}_z^\top, \delta\boldsymbol{\xi}_a^{1\top},..., \delta\boldsymbol{\xi}_a^{n\top}, \notag \\ & \quad \ \ \delta\boldsymbol{\xi}_c^{1\top},..., \delta\boldsymbol{\xi}_c^{n\top}] \in \mathbb{R}^{12n+18}.
\end{align}
Here, $\delta\boldsymbol{\xi}$ is the increment vector of the stacked calibration parameters, and $\boldsymbol{J}_i$ is the measurement-specific Jacobian matrix which maps the increments of the parameters to the residual vector. 
Deriving the explicit structure of $\boldsymbol{J}_i$ requires differentiating the lengthy PoE chain in \eqref{eq:B_PoE}, in other words, expanding the \eqref{eq:error_derivation}. In fact, the Lie derivative allows us to compute the differential of the exponential map in a closed form as follows~\cite{Park-Book1994}: For a coordinate twist $\hat{\boldsymbol{\xi}} \in se(3)$ and its increment $\delta\hat{\boldsymbol{\xi}}$, we have
\begin{align}
[ \delta\exp(\hat{\boldsymbol{\xi}}) \exp(-\hat{\boldsymbol{\xi}}) ]^\vee &= \mathcal{J}(\boldsymbol{\xi})\delta\boldsymbol{\xi}, 
\end{align}
where $\mathcal{J}(\boldsymbol{\xi})\in \mathbb{R}^{6\times6}$ denotes the Jacobian matrix of an element belonging to $se(3)$. As to a kinematic twist $\hat{\boldsymbol{\xi}}^k \in se(3)$ associated with a joint variable $q^k \in [-\pi, \pi]$, there is
\begin{align}
    [ \delta\exp(\hat{\boldsymbol{\xi}}^k q_k) \exp(-\hat{\boldsymbol{\xi}}^k q_k) ]^\vee &= \mathfrak{J}(\boldsymbol{\xi}^k, q_k)\delta\boldsymbol{\xi}^k, 
\end{align}
where $\mathfrak{J}(\boldsymbol{\xi}^k, q_k)\in \mathbb{R}^{6\times6}$ denotes the Jacobian matrix related to a kinematic twist~\cite{He-TRO2010}. 
The explicit formulas of the two Jacobian matrices $\mathcal{J}(\cdot)$ and $\mathfrak{J}(\cdot)$ are given in the appendix for clarity.
Using the above differential rules, now we can compute each block of $\boldsymbol{J}_i$ corresponding to each particular parameter.

We proceed to consider two kinds of parameters separately:
\begin{enumerate}[label=(\alph*)]
\item \textbf{Jacobian blocks of coordinate parameters}. 
These correspond to the twists $\boldsymbol{\xi}_x$, $\boldsymbol{\xi}_y$, and $\boldsymbol{\xi}_z$ describing the static transformations $\mathbf{X}$, $\mathbf{Y}$, and $\mathbf{Z}$. Since they appear in the PoE chain of \eqref{eq:B_PoE} without an accompanying joint variable, their differentials follow the coordinate twist rule. However, their contributions to the end-effector error must be transported through the preceding transformations via the adjoint map. For instance, the block for $\boldsymbol{\xi}_x$ is derived from the term $\exp(-\hat{\boldsymbol{\xi}}_x)$ at the beginning of the chain. Its differential yields $[ \delta\exp(-\hat{\boldsymbol{\xi}}_x) \exp(\hat{\boldsymbol{\xi}}_x) ]^\vee = - \mathcal{J}(-\boldsymbol{\xi}_x) \delta\boldsymbol{\xi}_x$, and since no transformation precedes it, the Jacobian block is simply 
\begin{align}\label{eq:J_x}
    \boldsymbol{J}_x^{i} = -\mathcal{J}(-\boldsymbol{\xi}_x).
\end{align}
As to $\boldsymbol{\xi}_y$, which appears after the kinematic chain of the sensor‑side robot, the differential of $\exp(\hat{\boldsymbol{\xi}}_y)$ contributes $\mathcal{J}(\boldsymbol{\xi}_y) \delta\boldsymbol{\xi}_y$, but this contribution must be left‑multiplied by the adjoint of all transformations before $\exp(\hat{\boldsymbol{\xi}}_y)$ in the chain. Hence, there is
\begin{align}\label{eq:J_y}
\boldsymbol{J}_y^{i} = \text{Ad}\big( &\exp(-\hat{\boldsymbol{\xi}}_x) \exp(-\hat{\boldsymbol{\xi}}_a^{st}) \exp(-\hat{\boldsymbol{\xi}}^n_aq^n_{a_i})\cdots \notag \\ & \times \exp(-\hat{\boldsymbol{\xi}}^1_aq^1_{a_i}) \big) \mathcal{J}(\boldsymbol{\xi}_y).
\end{align}
The block for $\boldsymbol{\xi}_z$ is obtained similarly, taking into account all transformations that precede $\exp(\hat{\boldsymbol{\xi}}_z)$, i.e., 
\begin{align}\label{eq:J_z}
    \boldsymbol{J}_z^i = \text{Ad}\big( &\exp(-\hat{\boldsymbol{\xi}}_x) \exp(-\hat{\boldsymbol{\xi}}_a^{st}) \exp(-\hat{\boldsymbol{\xi}}^n_aq^n_{a_i})\cdots \notag \\ & \times \exp(-\hat{\boldsymbol{\xi}}^1_aq^1_{a_i}) \exp(\hat{\boldsymbol{\xi}}_y)\exp(\hat{\boldsymbol{\xi}}^1_cq^1_{c_i})\cdots \notag \\ &\times \exp(\hat{\boldsymbol{\xi}}^n_cq^n_{c_i})\exp(\hat{\boldsymbol{\xi}}^{st}_c)  \big) \mathcal{J}(\boldsymbol{\xi}_z).
\end{align}

\item  \textbf{Jacobian blocks of kinematic parameters}.
These encompass the joint twists $\{\boldsymbol{\xi}_a^k\}_{k=1}^{n}$ and $\{\boldsymbol{\xi}_c^k\}_{k=1}^{n}$. 
Their differentials often involve the joint variable $q$, therefore
we need to use the kinematic Jacobian $\mathfrak{J}$. Consider the $k$-th~($k = 1, 2, ..., n$) joint twist of the sensor-side robot: the term $\exp(-\hat{\boldsymbol{\xi}}_a^k q_{a_i}^k)$ contributes $\mathfrak{J}(-\boldsymbol{\xi}_a^k, q_{a_i}^k) \delta\boldsymbol{\xi}_a^k$, and this contribution is transported by the adjoint of all transformations preceding this joint. Thus, the related Jacobian block is
\begin{align}\label{eq:J_ai}
    \boldsymbol{J}_{\xi_{a}^k}^{i} = - \text{Ad}\big( &\exp(-\hat{\boldsymbol{\xi}}_x) \exp(-\hat{\boldsymbol{\xi}}^{st}_a)\exp(-\hat{\boldsymbol{\xi}}^n_aq^n_{a_i})\cdots \notag \\ 
    & \times \exp(-\hat{\boldsymbol{\xi}}^{k+1}_aq^{k+1}_{a_i}) \big) \mathfrak{J}(-\boldsymbol{\xi}_a^k, q_{a_i}^k).
\end{align}
The blocks for the camera‑side joint twists are derived analogously, noting that their exponentials appear with positive signs in the chain. For the $k$-th~($k = 1, 2, ..., n$) joint twist of the chessboard-side robot, there is
\begin{align}\label{eq:J_ci}
    \boldsymbol{J}_{\xi_{c}^k}^{i} = \text{Ad}\big( &\exp(-\hat{\boldsymbol{\xi}}_x) \exp(-\hat{\boldsymbol{\xi}}^{st}_a)\exp(-\hat{\boldsymbol{\xi}}^n_aq^n_{a_i})\cdots \notag \\ 
    & \times \exp(-\hat{\boldsymbol{\xi}}^{1}_aq^{1}_{a_i}) \exp(\hat{\boldsymbol{\xi}}_y)\exp(\hat{\boldsymbol{\xi}}^1_cq^1_{c_i})\cdots \notag \\
    &\times \exp(\hat{\boldsymbol{\xi}}^{k-1}_cq^{k-1}_{c_i})   \big) \mathfrak{J}(\boldsymbol{\xi}_c^k, q_{c_i}^k).
\end{align}
\end{enumerate}

With all Jacobian blocks explicitly determined, we can assemble the complete $\boldsymbol{J}_i$ in \eqref{eq:devia_Jaco} for each measurement. Now, relying on \eqref{eq:devia_Jaco}, we can follow~\cite{He-TRO2010, Wu-TASE2014} to reduce the nonlinear least-squares problem in \eqref{eq:prob_quadratic} to a linear least‑squares problem:
\begin{align}
\min_{\delta\boldsymbol{\xi}} \sum_{i=1}^m \| \boldsymbol{e}_i - \boldsymbol{J}_i \delta\boldsymbol{\xi} \|_F^2,
\end{align}
where $\boldsymbol{e}_i = \log(\boldsymbol{B}_i'\boldsymbol{B}_i^{*-1})$ is the calibration error that can be easily computed. By stacking all measurement-specific errors and Jacobians into to a global residual vector $\boldsymbol{e} = [\boldsymbol{e}_1^\top, \boldsymbol{e}_2^\top, ..., \boldsymbol{e}_m^\top]^\top \in \mathbb{R}^{6m}$ and a global Jacobian matrix $\boldsymbol{J} = [\boldsymbol{J}_1^\top, \boldsymbol{J}_2^\top, ..., \boldsymbol{J}_m^\top]^\top \in \mathbb{R}^{6m\times(12n+18)}$, we can solve the following linear system to obtain the optimal increment $\delta\boldsymbol{\xi}^*$:
\begin{align}\label{eq:linear_system}
 \boldsymbol{e} = \boldsymbol{J}\delta\boldsymbol{\xi}^*.
\end{align}

The above formulation can be directly embedded into an iterative optimization algorithm, such as the Gauss‑Newton (G‑N) or Levenberg‑Marquardt (L‑M) methods~\cite{Nocedal-Book2006}. Concretely, at each iteration $t$, given the current parameter estimate $\boldsymbol{\xi}^{(t)}$, we compute the current global error $\boldsymbol{e}^{(t)}$ and Jacobian $\boldsymbol{J}^{(t)}$. Next, the increment $\delta\boldsymbol{\xi}^{*(t)}$ is solved from the linear system~\eqref{eq:linear_system}, and the parameters are updated via the exponential mapping:
\begin{align}
\exp(\hat{\boldsymbol{\xi}}^{(t+1)}) = \exp(\hat{\boldsymbol{\xi}}^{(t)}) \exp(\delta\hat{\boldsymbol{\xi}}^{*(t)}).
\end{align}
This update ensures that the calibration parameters remain on the SE(3) manifold space at each iteration.
Then, the iteration continues until the magnitude of the increment vector $\delta\boldsymbol{\xi}^{*(t)}$ falls below sufficiently small thresholds.

Up to now, we have established a unified Lie-algebraic error model and derived its closed-form analytical Jacobian, together with a well-conditioned iterative solving process for jointly updating all coordinate and kinematic parameters. 
Nevertheless, beyond numerical solvability, a crucial issue in parameter identification is whether the parameters are identifiable—i.e., whether the estimation can converge reliably, instead of being compromised by intrinsic redundancy. 
Such a kind of analysis is rarely explored in previous joint optimization methods~\cite{Wang-TMECH2024, Chen-TMECH2025}.
To this end, we next analyze the identifiability of the calibration parameters in our unified optimization framework.

\subsection{Identifiability Analysis}

Under the Jacobian-guided iterative solving scheme, a parameter is said to be \emph{unidentifiable} if the corresponding column of the identification Jacobian can be expressed as a linear combination of other columns; equivalently, the Jacobian matrix is not of full column rank. Therefore, we can access the identifiability of the calibration parameters by analyzing the rank properties of the Jacobian blocks in \eqref{eq:J_x}-\eqref{eq:J_ci}.

Before our analysis, we introduce two helpful standard properties~\cite{He-TRO2010}: 1) \textit{Non-singularity of adjoint operation}: For any $\mathbf{G}\in$ SE(3), $\mathrm{Ad}(\mathbf{G})\in\mathbb{R}^{6\times 6}$ is non-singular;
2) \textit{Non-singularity of $\mathcal{J}$ and $\mathfrak{J}$ under generic configurations}: $\mathcal{J}(\boldsymbol{\cdot})$ and $\mathfrak{J}(\cdot)$ are always non-singular except a few special cases, e.g., the joint variable in $\mathfrak{J}(\cdot)$ approaches 0. Next, our analysis proceeds in two steps: we first study each Jacobian block in \eqref{eq:J_x}--\eqref{eq:J_ci} individually, and then consider possible cross dependencies in the stacked Jacobian.

\textbf{Analysis on individual Jacobian blocks}. Consider the Jacobian blocks associated with the coordinate twists $\boldsymbol{\xi}_x,\boldsymbol{\xi}_y,\boldsymbol{\xi}_z$.
From \eqref{eq:J_x}--\eqref{eq:J_z}, each coordinate block can be written as
\begin{align}
 \boldsymbol{J}_{(\cdot)}^{i}=\pm \mathrm{Ad}(\boldsymbol{G}_{(\cdot)}^{i})\,\mathcal{J}(\pm\boldsymbol{\xi}_{(\cdot)})
\end{align}
where $\boldsymbol{G}_{(\cdot)}^{i}\in SE(3)$ denotes the preceding product in \eqref{eq:B_PoE} (possibly the identity).
By the two properties above, $\mathrm{Ad}(\boldsymbol{G}_{(\cdot)}^{i})$ is non-singular and $\mathcal{J}(\cdot)$ is non-singular under generic configurations.
Therefore, each coordinate block $\boldsymbol{J}_x^{i}$, $\boldsymbol{J}_y^{i}$ and $\boldsymbol{J}_z^{i}$ is of full column rank.
The same reasoning applies to the kinematic twists $\{\boldsymbol{\xi}_a^k\}_{k=1}^{n}$ and $\{\boldsymbol{\xi}_c^k\}_{k=1}^{n}$.
From \eqref{eq:J_ai} and \eqref{eq:J_ci}, each joint block admits the unified form
\begin{align}
    \boldsymbol{J}_{\xi}^{i}= \pm \mathrm{Ad}(\boldsymbol{G}_{\xi}^{i})\,\mathfrak{J}(\pm\boldsymbol{\xi},q),
\end{align}
where $\boldsymbol{G}_{\xi}^{i}$ is the product of transformations preceding the corresponding joint exponential in \eqref{eq:B_PoE}.
Again, $\mathrm{Ad}(\boldsymbol{G}_{\xi}^{i})$ is always non-singular, and $\mathfrak{J}(\cdot)$ is non-singular under generic configurations.
Hence, for any measurement where the joint value does not approach degenerate cases (e.g., $q\not\approx 0$), each joint block is of full column rank.

\textbf{Analysis on the stacked Jacobian}.  Having established that each individual Jacobian block is full-rank under generic configurations, we finally consider whether there exists a \emph{cross} dependence between coordinate and joint twists when stacking all measurements.
Let $\boldsymbol{J}\in\mathbb{R}^{6m\times(12n+18)}$ be the stacked Jacobian, and suppose there exists a nonzero increment vector
$\delta\boldsymbol{\xi}=[\delta\boldsymbol{\xi}_x^\top,\delta\boldsymbol{\xi}_y^\top,\delta\boldsymbol{\xi}_z^\top,\delta\boldsymbol{\xi}_a^\top,\delta\boldsymbol{\xi}_c^\top]^\top\neq\mathbf{0}$
such that $\boldsymbol{J}\delta\boldsymbol{\xi}=\mathbf{0}$.
Equivalently, for every measurement $i$,
\begin{align}\label{eq:cross_dep}
\boldsymbol{J}_x^{i}\delta\boldsymbol{\xi}_x+\boldsymbol{J}_y^{i}\delta\boldsymbol{\xi}_y+\boldsymbol{J}_z^{i}\delta\boldsymbol{\xi}_z
+\sum_{k=1}^{n}\boldsymbol{J}_{\xi_a^k}^{i}\delta\boldsymbol{\xi}_a^{k}
+\sum_{k=1}^{n}\boldsymbol{J}_{\xi_c^k}^{i}\delta\boldsymbol{\xi}_c^{k}
=\mathbf{0}.
\end{align}
The key observation is that the terms in \eqref{eq:cross_dep} are transported by different configuration-dependent adjoint factors.
Specifically, $\boldsymbol{J}_y^{i}$ varies with $i$ through $\mathrm{Ad}(\boldsymbol{G}_y^{i})$, $\boldsymbol{J}_z^{i}$ is transported by a longer preceding product,
and each joint block is transported by a truncation-dependent preceding product that differs across joint indices and measurements.
Therefore, a fixed nonzero $\delta\boldsymbol{\xi}$ cannot satisfy \eqref{eq:cross_dep} for all $i$ unless the dataset is degenerate, i.e., the motions do not provide sufficient excitation so that these transported adjoint factors collapse to identical (or nearly identical) variations across all measurements.
Under nondegenerate excitation with sufficiently diverse joint configurations and with each joint actuated away from degenerate values, \eqref{eq:cross_dep} admits only the trivial solution, and thus the stacked Jacobian $\boldsymbol{J}$ is full column rank.

The above analysis indicates that,
the increment vector of the calibration parameters, i.e., $\delta\boldsymbol{\xi}$, is identifiable under the following mild excitation conditions:
\begin{enumerate}[label=(\alph*)]
\item the collected measurements should cover \textit{multiple distinct dual-robot posture configurations}, such that the transported adjoint factors in the Jacobian blocks (e.g., $\mathrm{Ad}(\mathbf{G}_y^{i})$ and the truncation-dependent $\mathrm{Ad}(\mathbf{G}_{\xi}^{i})$) vary across samples rather than remaining constant;
\item the robot joints should be sufficiently excited such that \textit{each joint variable does not stay near degenerate values (e.g., $q^k\approx 0$) throughout the measurement dataset}; this ensures that the corresponding $\mathfrak{J}(\boldsymbol{\xi}^k,q^k)$ remains sensitive to the residual.
\end{enumerate}
These conditions are easy to satisfy in practice by collecting samples over diverse dual-robot posture configurations, and they ensure that the stacked identification Jacobian is of full column rank, thereby ensuring a well-posed and stable optimization in our unified calibration.

\section{SDP Relaxation for Certifiable Coordinate Initialization}
\label{sec:SDP_coord}

Since the coordinate transformations are unknown a priori, obtaining reliable initial estimates is essential before performing the joint optimization introduced above. 
This becomes especially critical under significant kinematic errors in the robot‑provided poses $\boldsymbol{A}_i$ and $\boldsymbol{C}_i$, since a poor initialization may cause slow convergence or trapping in suboptimal local minima. 
Motivated by this, we develop a certifiably correct initialization method for the coordinate parameters. Specifically, we first formulate the coordinate-only calibration problem as a Quadratically Constrained Quadratic Program~(QCQP) and then derive its convex SDP relaxation. This formulation enables us to compute a high-quality coordinate estimate together with an \textit{a posteriori} certificate of solution quality, thereby providing a robust starting point for the subsequent unified optimization.

\subsection{Quadratically Constrained Quadratic Programming}
\label{subsec:qcqp}
In this section, we formulate the coordinate parameter initialization problem and recast it into a standard QCQP.
Given the basic dual‑robot pose chain in~\eqref{eq:AXB=YCB}, we begin with isolating the rotation and translation constraints in the noise-free case as:
\begin{align}
    \boldsymbol{f}_i &\doteq \textnormal{vec}(\boldsymbol{R}_{a_i} \boldsymbol{R}_{x} \boldsymbol{R}_{b_i}) - \textnormal{vec}(\boldsymbol{R}_{y} \boldsymbol{R}_{c_i} \boldsymbol{R}_{z}) = \mathbf{0}_9, \\
    \boldsymbol{g}_i &\doteq \boldsymbol{R}_{a_i}\boldsymbol{R}_{x}\boldsymbol{t}_{b_i} + \boldsymbol{R}_{a_i}\boldsymbol{t}_x + \boldsymbol{t}_{a_i} \notag \\ & \quad \ - \boldsymbol{R}_y\boldsymbol{R}_{c_i}\boldsymbol{t}_z - \boldsymbol{R}_y\boldsymbol{t}_{c_i} - \boldsymbol{t}_y = \boldsymbol{0}_3. 
\end{align}

Accordingly, the following minimization problem can be defined for solving the initial coordinate parameters:
\begin{equation}\label{eq:coord_ori}
\begin{aligned}
     \min_{\substack{\boldsymbol{R}_{x}, \boldsymbol{R}_{y}, \boldsymbol{R}_{z}, \\ \boldsymbol{t}_x, \boldsymbol{t}_y, \boldsymbol{t}_z}} & \sum_{i = 1}^m ( \|\boldsymbol{f}_i\|_2^2 + \alpha^2 \|\boldsymbol{g}_i\|_2^2 ) \\
    \text{s.t.} \ & \boldsymbol{R}_{x}, \boldsymbol{R}_{y}, \boldsymbol{R}_{z} \in \text{SO}(3),\ i = {1, ..., m},
\end{aligned}
\end{equation}
where $\alpha$ is a predefined weighting coefficient. 

Next, we rewrite \eqref{eq:coord_ori} into a standard QCQP form
with a vectorized variable and constraints defined by real symmetric matrices. Our reformulation relies on
\begin{proposition} \label{prop:state_linear}
Define a homogeneous state vector as
\begin{align}
    \boldsymbol{w} \doteq [ &\textnormal{vec}(\boldsymbol{R}_x)^\top, \textnormal{vec}(\boldsymbol{R}_y)^\top, \textnormal{vec}(\boldsymbol{R}_z^\top\otimes\boldsymbol{R}_y)^\top, \notag \\ &  \boldsymbol{t}_x^\top, \boldsymbol{t}_y^\top, \textnormal{vec}(\boldsymbol{t}_z^\top\otimes\boldsymbol{R}_y)^\top, 1]^\top \in \mathbb{R}^{133}.
\end{align}
Then $\boldsymbol{f}_i$ and $\boldsymbol{g}_i$ are linear in $\boldsymbol{w}$, i.e.,
\begin{equation}
    \boldsymbol{f}_i = \boldsymbol{\Omega}_{f_i}\boldsymbol{w}, \quad \boldsymbol{g}_i = \boldsymbol{\Omega}_{g_i} \boldsymbol{w},
\end{equation}
where $\boldsymbol{\Omega}_{f_i},\ \boldsymbol{\Omega}_{g_i}$ are introduced as auxiliary matrices and constructed from the measurements.
\end{proposition}
The proof and the explicit expressions of $\boldsymbol{\Omega}_{f_i}$ and $\boldsymbol{\Omega}_{g_i}$ are given in the appendix for clarity. Now, we can rewrite each term in the objective function of \eqref{eq:coord_ori} as
\begin{align}
    \|\boldsymbol{f}_i\|_2^2 + \alpha\|\boldsymbol{g}_i\|_2^2 &= \boldsymbol{w}^\top\boldsymbol{\Omega}_{f_i}^\top\boldsymbol{\Omega}_{f_i}\boldsymbol{w} + \alpha^2\boldsymbol{w}^\top\boldsymbol{\Omega}_{g_i}^\top\boldsymbol{\Omega}_{g_i}\boldsymbol{w} \\
    &= \boldsymbol{w}^\top(\boldsymbol{Q}_{f_i} + \alpha^2\boldsymbol{Q}_{g_i})\boldsymbol{w} \\
    &= \boldsymbol{w}^\top\boldsymbol{Q}_{i}\boldsymbol{w}
\end{align}
where we define $\boldsymbol{Q}_{f_i} \doteq \boldsymbol{\Omega}_{f_i}^\top\boldsymbol{\Omega}_{f_i}$, $\boldsymbol{Q}_{g_i} \doteq \boldsymbol{\Omega}_{g_i}^\top\boldsymbol{\Omega}_{g_i}$, and $\boldsymbol{Q}_i \doteq \boldsymbol{Q}_{f_i} + \alpha^2\boldsymbol{Q}_{g_i}$. Note that $\boldsymbol{Q}_{f_i}$ and $\boldsymbol{Q}_{g_i}$ are inherently symmetric and positive semidefinite, therefore we can conclude that $\boldsymbol{Q}_i$ shares the same property. 

Now we consider the SO(3) constraints in \eqref{eq:coord_ori}:
\begin{proposition} \label{prop:SO3_QC}
    The SO(3) constraints in $\boldsymbol{R}_x$, $\boldsymbol{R}_y$, and $\boldsymbol{R}_z$ can be properly transformed into a set of quadratic constraints on the state vector $\boldsymbol{w}$, i.e.,
    \begin{align}
        \boldsymbol{w}^\top \boldsymbol{H}_j\boldsymbol{w} = \rho_j,\ \forall j = 1, 2, ..., h,
    \end{align}
    where each $\boldsymbol{H}_j$ is a real symmetric matrix, each $\rho_j$ is a constant, and $h$ is the number of constraints.  
\end{proposition}
The complete proof is left to the appendix.
To finish the final preparation, we define $\boldsymbol{Q} \doteq \sum_{i=1}^{m}\boldsymbol{Q}_i$. At this point we are ready to rewrite Problem~\eqref{eq:coord_ori} into the QCQP form:
\begin{equation}\label{eq:qcqp}
\begin{aligned}
    \min_{\boldsymbol{w}} & \ \boldsymbol{w}^\top\boldsymbol{Q}\boldsymbol{w} \\
    \textnormal{s.t.} \ & \boldsymbol{w}^\top \boldsymbol{H}_j\boldsymbol{w} = \rho_j, \ \forall j = 1, 2, ..., h, \\
    & \boldsymbol{w}^\top\boldsymbol{H}_{h+1}\boldsymbol{w} = 1,
\end{aligned} 
\end{equation}
where $\boldsymbol{H}_{h+1} = \mathrm{diag}(\mathbf{0}_{(132)\times(132)},\,1)$ is introduced to fix the the homogeneous coordinate introduced in $\boldsymbol{w}$. One may concern that $\boldsymbol{w}^\top \boldsymbol{H}_{h+1}\boldsymbol{w}=1$ admits two symmetric branches $\boldsymbol{w}_{(133)}=\pm 1$, potentially causing a sign ambiguity. In fact, this does not pose an issue here because the data-induced parts in $\boldsymbol{Q}$ is naturally able to break the $\boldsymbol{w}\mapsto-\boldsymbol{w}$ symmetry; besides, the lifting is used solely to encode affine terms (the appended constant entry) and all physically meaningful quantities are recovered from the remaining blocks of $\boldsymbol{w}$.

\subsection{Semidefinite Relaxation}
Directly solving problem~\eqref{eq:qcqp} remains challenging due to the non‑convex quadratic equality constraints. However, its QCQP structure presents a significant benefit, that is, we can reformulate it further and develop a convex semidefinite programming~(SDP) relaxation. 

The relaxation is built upon reparameterizing problem~\eqref{eq:qcqp} using the following state matrix:
\begin{align}
    \boldsymbol{W} = \boldsymbol{w}\boldsymbol{w}^\top,
\end{align}
which is by construction a rank‑1 symmetric positive semidefinite matrix. 
With this substitution, the quadratic objective function in problem~\eqref{eq:qcqp} can be expressed linearly in $\boldsymbol{W}$:
\begin{align}
    \boldsymbol{w}^\top\boldsymbol{Q}\boldsymbol{w} = \textnormal{tr}(\boldsymbol{w}^\top\boldsymbol{Q}\boldsymbol{w}) = \textnormal{tr}(\boldsymbol{Q}\boldsymbol{w}\boldsymbol{w}^\top) = \textnormal{tr}(\boldsymbol{Q}\boldsymbol{W}).
\end{align}
Similarly, each quadratic constraint in problem~\eqref{eq:qcqp} becomes $\boldsymbol{w}^\top\boldsymbol{H}_j\boldsymbol{w} = \textnormal{tr}(\boldsymbol{H}_j\boldsymbol{W}) = \rho_j$.
Consequently, problem~\eqref{eq:qcqp} is equivalent to the following rank‑constrained SDP:
\begin{equation}\label{eq:sdp_primal}
    \begin{aligned}
         \min_{\boldsymbol{W}}  &\ \textnormal{tr}(\boldsymbol{Q}\boldsymbol{W}) \\
    \textnormal{s.t.} \ & \textnormal{tr}(\boldsymbol{H}_j\boldsymbol{W}) = \rho_j, \ \forall j = 1, 2, ..., h + 1, \\
    & \boldsymbol{W} \succeq 0, \ \boldsymbol{W}^\top = \boldsymbol{W}, \\ &\textnormal{rank}(\boldsymbol{W}) = 1.
    \end{aligned}
\end{equation}
Here, $\boldsymbol{W} \succeq 0$ means that $\boldsymbol{W}$ is positive semidefinite.
In this formulation, the sole non‑convexity arises from the rank‑1 constraint. Removing it yields a convex SDP relaxation, stated formally below:
\begin{proposition}
The convex semidefinite program~(SDP)
\begin{equation}\label{eq:sdp}
    \begin{aligned}
         \min_{\boldsymbol{W}}  &\ \textnormal{tr}(\boldsymbol{Q}\boldsymbol{W}) \\
    \textnormal{s.t.} \ & \textnormal{tr}(\boldsymbol{H}_j\boldsymbol{W}) = \rho_j, \ \forall j = 1, 2, ..., h + 1, \\
    & \boldsymbol{W} \succeq 0, \ \boldsymbol{W}^\top = \boldsymbol{W}.
    \end{aligned}
\end{equation}
is a convex relaxation of problem~\eqref{eq:sdp_primal}.
\end{proposition}
Problem~\eqref{eq:sdp} possesses a key advantage: as a convex optimization problem, it can be solved with high efficiency and global optimality by numerous off-the-shelf tools, e.g., the solvers in \texttt{CVXPY} library~\cite{Diamond-JMLR2016}. 
However, a central concern is that our relaxation may not be tight, which could introduce a gap between the SDP optimum and the true optimum of the original problem.
Fortunately, the SDP framework provides a powerful mechanism for a \textit{posteriori} certification. The following section formalizes this certification and details the complete procedure for our solution recovery.

\subsection{Certifiably Correct Solution}

Let $\boldsymbol{W}^*$ denote the obtained optimal solution of the convex SDP \eqref{eq:sdp}. The critical subsequent steps are to recover a feasible candidate solution for the original QCQP in problem~\eqref{eq:qcqp} and to assess its quality. 
This process is grounded in the following verifiable tightness condition of our SDP relaxation.
\begin{theorem}[\textbf{Tightness Guarantee of Relaxation}]\label{theorem:tightness}
Let $\boldsymbol{W}^*$ be the optimal solution of problem~\eqref{eq:sdp}. If $\operatorname{rank}(\boldsymbol{W}^*) = 1$, then the SDP relaxation is \textit{tight}. In this case, $\boldsymbol{W}^* = \boldsymbol{w}^*\boldsymbol{w}^{*\top}$ for some vector $\boldsymbol{w}^*$, which constitutes a globally optimal solution to the original problem~\eqref{eq:qcqp}.
\end{theorem}
The detailed proof is given in the appendix. Note that this property of SDP relaxation is well established in the optimization areas~\cite{Bandeira-CRM2016, Cifuentes-MP2022, Rosen-IJRR2019}.
Back to our analysis, if the solved $\boldsymbol{W}^*$ admits the rank-1 constraint, then a global minimizer $\boldsymbol{w}^*$ of the original QCQP \eqref{eq:qcqp} can be extracted directly via Cholesky decomposition. 
In practice, particularly under significant measurement noise, $\boldsymbol{W}^*$ may not
strictly satisfy the rank-1 condition. 
In this case, we propose to compute a candidate $\boldsymbol{w}$ from the best rank-1 approximation of $\boldsymbol{W}^*$. Then the recovered candidate is projected onto the original constraint manifold of the coordinate parameters ${\boldsymbol{X}, \boldsymbol{Y}, \boldsymbol{Z}}$ to ensure strict feasibility~(e.g., via orthogonal projection to $\operatorname{SO}(3)$ for rotation components).

For simplicity, we denote the above solved state vector by $\boldsymbol{w}^{\star}$.
To certify the quality of $\boldsymbol{w}^{\star}$, we exploit the property that the SDP optimal value $p^*_{\text{SDP}} = \operatorname{tr}(\boldsymbol{Q}\boldsymbol{W}^*)$ provides a global lower bound for the original non-convex problem~\eqref{eq:qcqp}. The sub-optimality gap of $\hat{\boldsymbol{w}}$ can thus be quantified as:
\begin{align}
    \eta = \frac{\boldsymbol{w}^{\star\top}\boldsymbol{Q}\boldsymbol{w}^\star - p^*_{\text{SDP}}}{p^*_{\text{SDP}}}.
\end{align}
A small value of $\eta$ certifies that $\boldsymbol{w}^\star$ is nearly optimal. This a posteriori certification is pivotal: even when the relaxation is not tight ($\eta \gg 0$), a small $\eta$ guarantees that the initial coordinate solution is reliable and lies within a strong basin of attraction for the subsequent holistic optimization, thereby ensuring robust convergence against large kinematic noise.

In fact, in our simulation and real-data experiments~(see Section~\ref{sec:exper}), \textit{we observe that $\boldsymbol{W}^\star$ satisfies the rank-one condition in the vast majority of test cases}, i.e., $\operatorname{rank}(\boldsymbol{W}^\star)=1$.
This empirical evidence demonstrates that the proposed SDP relaxation is typically tight for most cases.
In the rare cases where $\boldsymbol{W}^\star$ is not exactly rank-one (most notably under very large measurement noise), we still consistently obtain a very small sub-optimality gap, e.g., $\eta < 10^{-3}$.
This indicates that the recovered $\boldsymbol{w}^\star$ remains near-globally optimal and hence provides a high-quality and trustworthy initialization.
Overall, these observations corroborate that our SDP-based initialization solver offers not only empirical tightness in most scenarios, but also a reliable \textit{a posteriori} certificate in challenging regimes, thereby furnishing a robust starting point for our joint optimization.

\section{Experiment Results}
\label{sec:exper}

To verify the effectiveness of the proposed calibration framework, we compare our approach with several baselines and state-of-the-arts regarding dual-arm robot calibration. The evaluation consists of both simulation analysis~(see Section~\ref{subsec:sim_exper}) and real-world experiments~(see Section~\ref{subsec:real_exper}).

The compared methods include three coordinate-only methods and one coordinate-kinematic joint calibration method~(\textit{in this paper, the joint methods are marked with \#}) that all take identical visual measurements as our approach:
\begin{enumerate}[label=(\alph*)]
    \item \textsf{Wu}~\cite{Wu-TRO2016} that involves quaternion-based rotation initialization and linear approximation iterative optimization.
    \item \textsf{Wang}~\cite{Wang-TRO2021} that initializes via Kronecker products and conducts iterative refinement with orthogonal constraints.
    \item \textsf{Jiang}~\cite{Jiang-TRO2023} that constructs a convex optimization problem using Lie algebra representations of the coordinates.
    \item \textsf{Chen}\#~\cite{Chen-TMECH2025} that incorporates kinematic errors via an arm-wise decoupled error model and employs a hierarchical identification technique for calibration. 
\end{enumerate}
As our framework involves a coordinate-only initialization, for comprehensive evaluation, we denote our initialization part by \textsf{SDP-Ini} and the whole approach by \textsf{Unified}\#.

\begin{figure}[!t]
    \centering
    \includegraphics[width=0.42\textwidth]{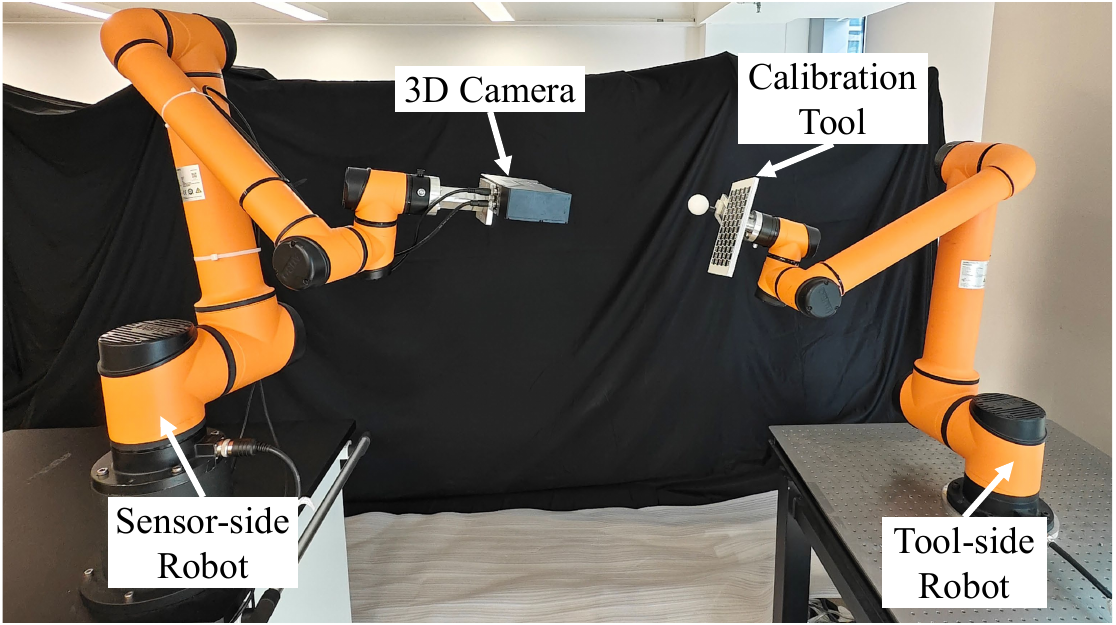}
    \\[-0.4em]
    \caption{Real-world experiments are conducted on a dual-arm robot system that contains an AUBO-i10 arm equipped with an industrial 3D camera and an AUBO-i10 arm equipped with a calibration tool. The calibration tool consists of a high-precision ChArUco chessboard and a standard ceramic ball.} 
    \label{fig:real_system}
\end{figure}

\textbf{Implementation Details}.
All compared methods are implemented using the provided code or faithfully re-implemented following the original papers.
In the coordinate initialization part of our algorithm, we set the weighting coefficient $\alpha$ in \eqref{eq:coord_ori} to 1.0, where the translational term is measured in meters. 
Besides, to solve the SDP, we employ the Splitting Conic Solver~(SCS)~\cite{Donoghue-JOTA2016} in the \texttt{CVXPY} library~\cite{Diamond-JMLR2016}.
In the subsequent unified optimization, we employ a damped Gauss--Newton solving scheme
with a small damping factor $\lambda=10^{-3}$ to improve numerical stability.
Our joint optimization terminates when $\|\delta\boldsymbol{\xi}^{*(t)}\|_{\infty} \le 10^{-3}$, with a maximum of $100$ iterations to avoid rare ill-conditioned cases.
For simulation analysis in the experiments, we employ the PyBullet simulator~\cite{Coumans-pybullet2016} and construct a virtual dual-robot system with two UR5 Arms by the Universal Robots~(see Fig.~\ref{fig:ill_dualrobot}). All ground-truth system parameters in the simulated robot system can be easily obtained.
In real-world experiments, practical calibration is conducted on a dual-robot platform with two AUBO-i10 arms with measured repeatability $\pm 0.03$ mm by the AUBO Robotics~(see Fig.~\ref{fig:real_system}).
The sensor-side robot arm is equipped with an industrial 3D camera that captures 2448$\times$2048 rectified grayscale images and provides point clouds with a measurement accuracy of $\pm$0.03 mm at a reference distance of 0.4 m.
The tool-side robot arm is equipped with a calibration tool consisting of a high-precision ChArUco board and a standard ceramic ball with known diameter $d = 50.8228$ mm. 
When collecting the calibration data, we use the captured chessboard images to estimate the tool pose in the sensor frame, denoted as $\boldsymbol{B}_i^{*}$. For evaluation, we employ additional chessboard images and point clouds of the standard ball. 
All the experiments are conducted on a laptop equipped with an Intel Core i7-10875H CPU@2.3 GHz, and 32GB of RAM.

\begin{table}[!t]
\centering
    \caption{Distribution of robot end-effector pose error caused by different levels of kinematic errors in the simulation analysis~(c.f., Section~\ref{subsec:sim_exper}).}
    \vspace{-0.5em}
    \label{tab:error_kine}
    \footnotesize
     \renewcommand{\tabcolsep}{2pt} 
    \renewcommand\arraystretch{1.2}
     \begin{tabular}{c|c|c}
         \Xhline{1pt}
         \multirow{2}{*}{\makecell{Kinematic \\ Error Level}} & Rotation Deviation~(deg.)  & Translation Deviation~(mm)   \\
         & mean $\pm$ std & mean $\pm$ std \\
         \Xhline{0.5pt}
         low~(\textit{L}) & 0.054 $\pm$ 0.022 & 0.417 $\pm$ 0.161 \\
         medium-low~(\textit{ML}) & 0.103 $\pm$ 0.043 & 1.083 $\pm$ 0.412 \\
         medium~(\textit{M}) & 0.264 $\pm$ 0.089 & 1.818 $\pm$ 0.763 \\
         medium-high~(\textit{MH})  & 0.427 $\pm$ 0.143 & 2.861 $\pm$ 1.032 \\
         high~(\textit{H})  & 0.692 $\pm$ 0.278 & 5.567 $\pm$ 0.238 \\
         quite-high~(\textit{QH})  & 1.423 $\pm$ 0.509 & 8.297 $\pm$ 3.659 \\
         \Xhline{1pt}
     \end{tabular}
\end{table}

\begin{table}[!t]
\centering
    \caption{Distribution of measurement noise at different levels in the simulation analysis~(c.f., Section~\ref{subsec:sim_exper}).}
    \vspace{-0.5em}
    \label{tab:noise_mea}
    \footnotesize
     \renewcommand{\tabcolsep}{5.3pt} 
    \renewcommand\arraystretch{1.2}
     \begin{tabular}{c|c|c}
         \Xhline{1pt}
         \multirow{2}{*}{\makecell{Measurement \\ Noise Level}} & Rotation Noise~(deg.)  & Translation Noise~(mm)   \\
         & mean $\pm$ std & mean $\pm$ std \\
         \Xhline{0.5pt}
         low~(\textit{L}) & 0.032 $\pm$ 0.013 & 0.080 $\pm$ 0.034\\
         medium-low~(\textit{ML}) & 0.080 $\pm$ 0.034 & 0.160 $\pm$ 0.067 \\
         medium~(\textit{M}) & 0.128 $\pm$ 0.054 & 0.479 $\pm$ 0.202 \\
         medium-high~(\textit{MH})  & 0.160 $\pm$ 0.067 & 0.798 $\pm$ 0.337 \\
         high~(\textit{H})  & 0.239 $\pm$ 0.101 & 1.277 $\pm$ 0.539 \\
         quite-high~(\textit{QH})  & 0.319 $\pm$ 0.135 & 1.596 $\pm$ 0.673 \\
         \Xhline{1pt}
     \end{tabular}
\end{table}

\subsection{Simulation Analysis}
\label{subsec:sim_exper}

\textbf{Data Generation}.
In the simulator, we synchronically adjust the attitudes of both robots within their respective workspaces. 
The invalid robot attitudes, e.g., the joint variable of any robot joint stays near zero points, should be deleted.
We further enforce sample diversity by requiring sufficiently large pose/joint differences between any two consecutive samples.
To simulate the kinematic error, we perturb the nominal kinematic parameters of each joint by a small twist increment and treat the perturbed parameters as ground truth. For example, for the $k$-th joint of the sensor-side robot, we define
\begin{align}
    \exp({}_{\mathrm{gt}}\boldsymbol{\xi}_a^k) = \exp({}_{\mathrm{nom}}\boldsymbol{\xi}_a^k)\exp( \Delta\boldsymbol{\xi}_a^k).
\end{align}
We similarly obtain $\{{}_{\mathrm{gt}}\boldsymbol{\xi}_c^{k}\}_{k=1}^{n}$ of the tool-side robot. Given $\{{}_{\mathrm{gt}}\boldsymbol{\xi}_a^{k}\}_{k=1}^{n}$, $\{{}_{\mathrm{gt}}\boldsymbol{\xi}_c^{k}\}_{k=1}^{n}$ and the preset static coordinate transformations $\{\boldsymbol{X},\boldsymbol{Y},\boldsymbol{Z}\}$, we generate the ground-truth relative measurements ${}_{\mathrm{gt}}\boldsymbol{B}$ through the dual-arm kinematic chain.
To simulate measurement noise, we further corrupt each ${}_{\mathrm{gt}}\boldsymbol{B}$ by a small Gaussian-distributed perturbation $\Delta\boldsymbol{B}$, i.e.,
\begin{align}
    \boldsymbol{B} = {}_{\mathrm{gt}}\boldsymbol{B} \Delta\boldsymbol{B}.
\end{align}
We design six discrete levels for both kinematic perturbations and measurement noise, categorized as low~(\textit{L}), medium-low~(\textit{ML}), medium~(\textit{M}), medium-high~(\textit{MH}), high~(\textit{H}), and quite-high~(\textit{QH}). 
For clarity, Table~\ref{tab:error_kine} and Table~\ref{tab:noise_mea} summarize the resulting pose disturbances (mean $\pm$ std) at each level in terms of rotational and translational components.
Note that for kinematic perturbations, we do not report the raw twist increments $\Delta\boldsymbol{\xi}$ directly; instead, we characterize each level by the induced robot end-effector pose error between the nominal and ground-truth forward kinematics.
This pose-level characterization is more interpretable and allows a more direct comparison.
In our settings, the first three levels (\textit{L} / \textit{ML} / \textit{M}) correspond to typical industrial conditions (moderate kinematic imperfections and common camera-based measurement accuracies), while the last three levels (\textit{MH} / \textit{H} / \textit{QH}) create more challenging scenarios.

\begin{figure}[!t]
    \centering
    \includegraphics[width=0.475\textwidth]{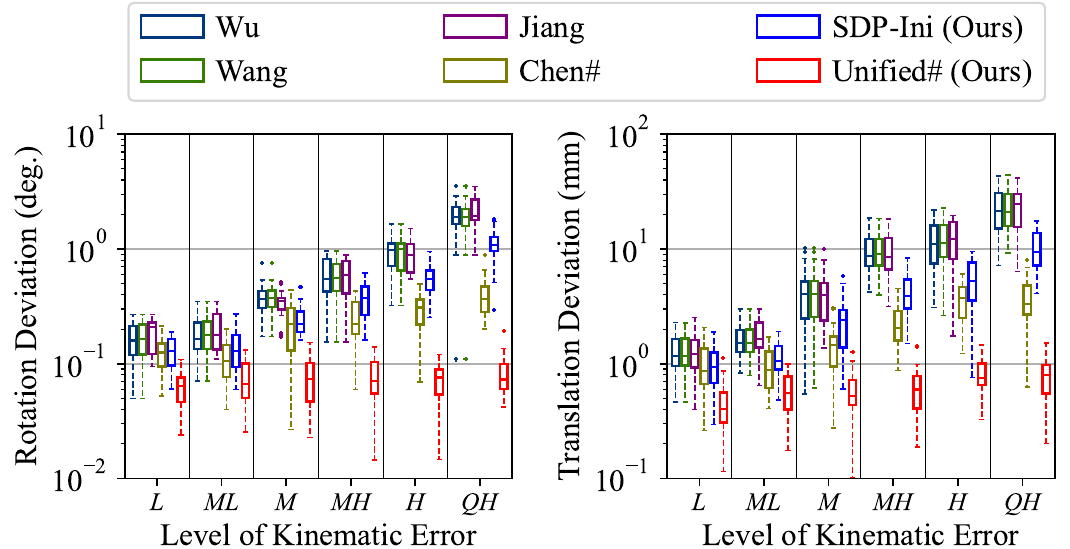}
    \\[-0.3em]
    \makebox[0.24\textwidth]{\footnotesize \quad \quad \quad \ \ (a-1)}
    \makebox[0.24\textwidth]{\footnotesize \quad \quad \ \ (a-2)}
    \\[0.1em]
    \includegraphics[width=0.475\textwidth]{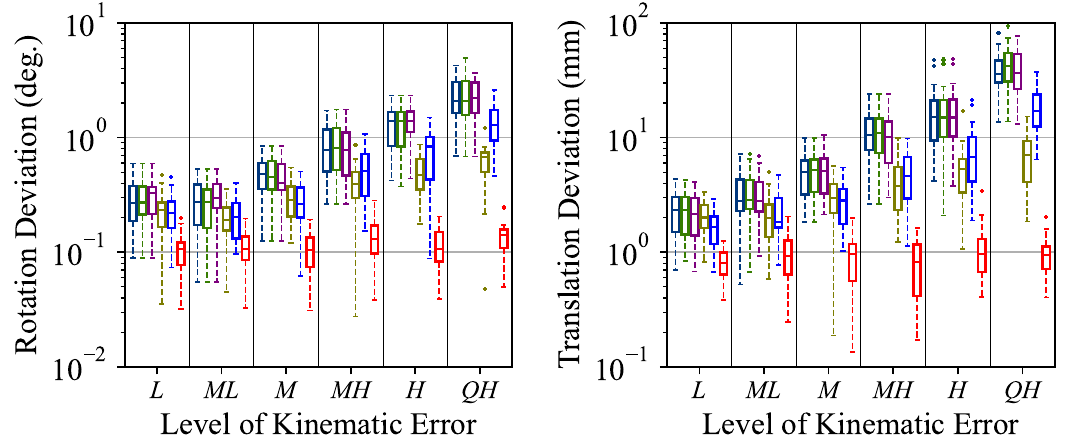}
    \\[-0.4em]
    \makebox[0.24\textwidth]{\footnotesize \quad \quad \quad \ (b-1)}
    \makebox[0.24\textwidth]{\footnotesize \quad \quad \ \ (b-2)}
    \\
    \vspace{-0.25em}
    \caption{Simulated evaluation results on different levels of kinematic errors under (a) medium and (b) high measurement noises~(see Section~\ref{subsec:sim_exper}1. Each box plot summarizes the error distribution over 40 distinct test samples.}
    \label{fig:kin_levels}
\end{figure}

\textbf{Evaluation Metrics}.
To evaluate both coordinate-only and joint coordinate-kinematic calibration methods under a unified criterion, we measure the \textit{closed-loop kinematic consistency} of the dual-arm transmission chain. For the $i$-th test sample, we define the closed-loop transformation error as
\begin{align}
    \boldsymbol{E}_i \doteq
(\boldsymbol{A}_i \boldsymbol{X}\boldsymbol{B}_i)^{-1} \boldsymbol{Y}\boldsymbol{C}_i\boldsymbol{Z}
\in \mathrm{SE}(3).
\label{eq:Ei_def}
\end{align}
$\boldsymbol{E}_i$ can be seen as the \emph{transformation deviation} that measures how much the calibrated chain violates the ideal loop equation $\boldsymbol{A}_i\boldsymbol{X}\boldsymbol{B}_i = \boldsymbol{Y}\boldsymbol{C}_i\boldsymbol{Z}$. 
For a \textit{coordinate-only} method, we use the calibration outputs ${\boldsymbol{X}, \boldsymbol{Y}, \boldsymbol{Z}}$, while $\boldsymbol{A}_i$ and $\boldsymbol{C}_i$ are directly obtained from the robot controllers (i.e., computed using the nominal forward kinematics with controller-provided kinematic parameters).
For a \textit{joint} calibration method, in addition to the optimized ${\boldsymbol{X}, \boldsymbol{Y}, \boldsymbol{Z}}$, we re-compute $\boldsymbol{A}_i$ and $\boldsymbol{C}_i$ using the newly estimated kinematic parameters together with the joint variables obtained from the robot controllers, so that the closed-loop error in \eqref{eq:Ei_def} evaluates the \emph{fully calibrated} system.
For better quantitative evaluation, we separate the transformation deviation $\boldsymbol{E}_i$ into rotational and translational components. By expanding $\boldsymbol{E}_i$ as
\begin{align}
    \boldsymbol{E}_i=
\begin{bmatrix}
\boldsymbol{R}_{i} & \boldsymbol{t}_{i} \\
\mathbf{0}^\top_3 & 1
\end{bmatrix},
\end{align}
we define the rotation and translation errors as
\begin{align}
e_{\boldsymbol{R}_i} & \doteq \angle(\boldsymbol{R}_{i})
=\arccos\left(\frac{\mathrm{tr}(\boldsymbol{R}_{i})-1}{2}\right),
\label{eq:rot_err} \\
e_{\boldsymbol{t}_i} &\doteq \|\boldsymbol{t}_{i}\|_2 ,
\label{eq:trans_err}
\end{align}
where $e_{\boldsymbol{R}_i}$ denotes the rotation deviation and $e_{\boldsymbol{t}_i}$ represents the translation deviation.

\vspace{0.5em}
\noindent \textit{1) Evaluation on Different Levels of Kinematic Errors:}
\vspace{0.2em}

We first evaluate how the competing calibration methods tolerate increasing robot kinematic errors.
Following the above data generation procedure,
we generate six levels of kinematic perturbations and test all methods under two representative measurement-noise settings: medium and high.
At each setup, we independently synthesize 80 measurement samples to perform the calibration,
and then generate another 40 fresh distinct samples for evaluation.
For each test sample, we compute the closed-loop deviation $\boldsymbol{E}_i$ in~\eqref{eq:Ei_def} and report its rotation/translation components
$e_{\boldsymbol{R}_i}$ and $e_{\boldsymbol{t}_i}$.
All statistics are aggregated and visualized as box plots in Fig.~\ref{fig:kin_levels}.

\textbf{Results}.
As shown in Fig.~\ref{fig:kin_levels}(a-1,a-2) for medium measurement noise and Fig.~\ref{fig:kin_levels}(b-1,b-2)  for high noise,
the performance of all \emph{coordinate-only} baselines~(\textsf{Wu}, \textsf{Wang}, \textsf{Jiang}) degrades consistently as the kinematic error level increases.
This trend is expected as the pose bias induced by kinematics cannot be compensated by optimizing coordinate parameters alone, leading to accumulated loop inconsistency. \textsf{Chen}\# is a \emph{joint} calibration method and therefore alleviates part of the kinematic-induced bias.
However, their arm-wise decoupled correction is not explicitly constrained by the full dual-arm transmission chain, and thus can easily enlarge the accumulated errors when propagated through the coupled pose loop, especially under high and quite-high kinematic perturbations.

In contrast, our \textsf{Unified}\# remains substantially more stable across all kinematic levels under both levels of measurement noise,
achieving the lowest average rotation and translation deviations.
The advantage becomes more pronounced at the high kinematic error regimes (\textit{H} and \textit{QH}),
where coordinate-only solutions exhibit rapid error growth, and Chen\# shows increasing dispersion.
It is also worth noting that, under the same measurement-noise level, the performance of \textsf{Unified}\# varies only mildly as the kinematic level increases.
This is because \textsf{Unified}\# explicitly models and jointly updates both arms' kinematic parameters together with the coordinate transformations under a unified loop-consistency objective, so that kinematic-induced systematic biases are largely absorbed by our kinematic parameter refinement rather than being accumulated along the dual-arm pose chain.

Moreover, our \textsf{SDP-Ini} consistently yields lower deviations than the coordinate-only baselines across all kinematic levels, indicating that the SDP relaxation provides a more robust coordinate initialization even when the robot-provided poses are biased by kinematic errors. 
Interestingly, when the kinematic error is relatively small (e.g., levels \textit{L} - \textit{M}), \textsf{SDP-Ini} and the joint calibration baseline \textsf{Chen}\# often achieve comparable deviations.
This suggests that in the mild kinematic perturbation regime, the dual-arm loop inconsistency can be reduced by either an accurate coordinate estimate alone or the additional kinematic updates in \textsf{Chen}\#. 
Nevertheless, our \textsf{Unified}\# consistently preserves its advantage and becomes increasingly superior as the kinematic errors grow.

Finally, comparing the medium and high measurement noise settings, all methods exhibit slightly degraded accuracy, which is expected when the measured $\boldsymbol{B}_i$ becomes less reliable.
Thus, we further analyze the impact of measurement noise in the next experiment.

\begin{figure}[!t]
    \centering
    \includegraphics[width=0.475\textwidth]{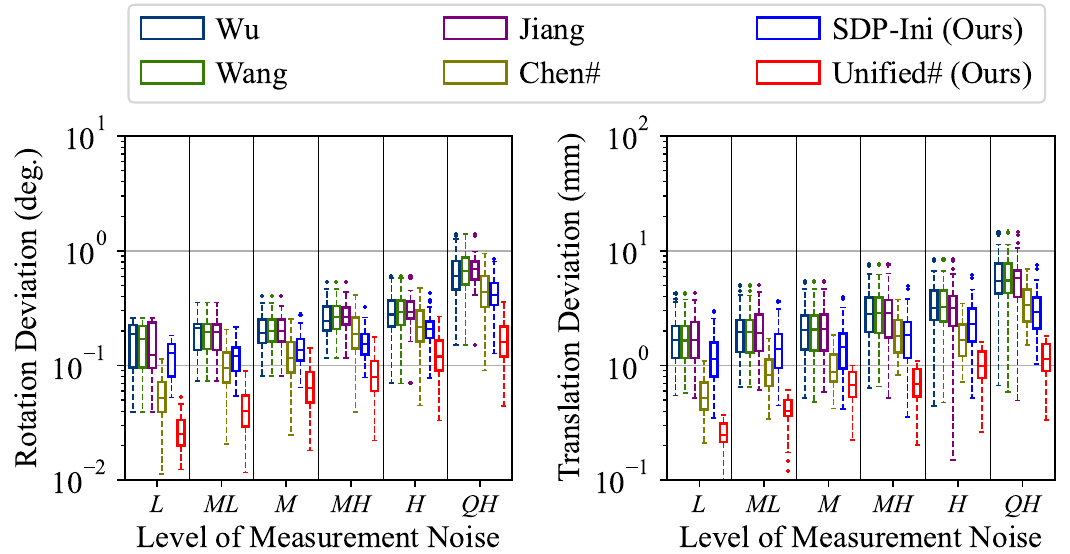}
    \\[-0.3em]
    \makebox[0.24\textwidth]{\footnotesize \quad \quad \quad \ \ (a-1)}
    \makebox[0.24\textwidth]{\footnotesize \quad \quad \ \ (a-2)}
    \\[0.1em]
    \includegraphics[width=0.475\textwidth]{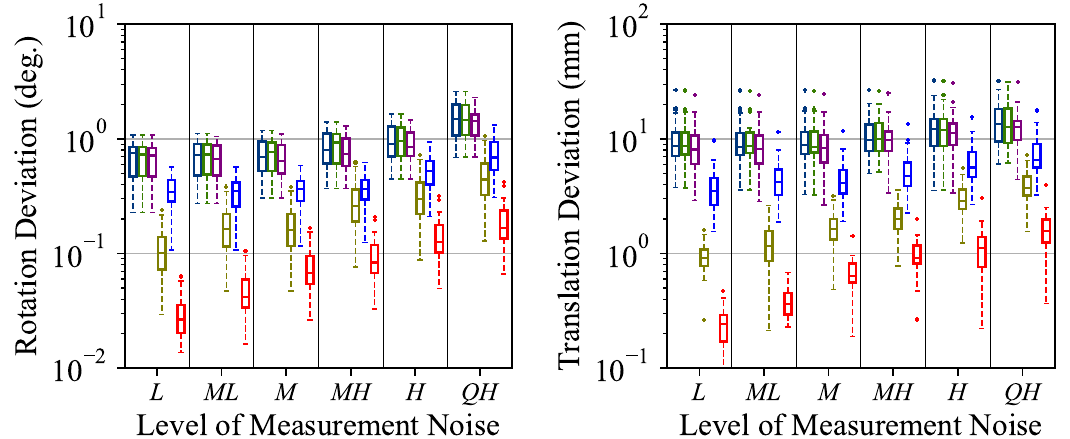}
    \\[-0.4em]
    \makebox[0.24\textwidth]{\footnotesize \quad \quad \quad \ (b-1)}
    \makebox[0.24\textwidth]{\footnotesize \quad \quad \ \ (b-2)}
    \\
    \vspace{-0.25em}
    \caption{Simulated evaluation results on different levels of measurement noises under (a) medium-low and (b) medium-high kinematic errors~(see Section~\ref{subsec:sim_exper}2). Each box plot summarizes the error distribution over 40 distinct test samples.}
    \label{fig:mea_levels}
\end{figure}

\vspace{0.5em}
\noindent \textit{2) Evaluation on Different Levels of Measurement Noises:}
\vspace{0.2em}

In this experiment, we sweep six measurement-noise levels in Table~\ref{tab:noise_mea} while fixing the kinematic perturbation to two representative settings: medium-low and medium-high.
At each setup, we generate 80 samples for calibration and another 40 unseen samples for testing,
and evaluate the closed-loop deviations using \eqref{eq:rot_err} and \eqref{eq:trans_err}.
The resulting distributions are summarized as box plots in Fig.~\ref{fig:mea_levels}.

\textbf{Results}.
Fig.~\ref{fig:mea_levels} shows that our \textsf{Unified}\# consistently achieves the smallest average rotation and translation deviations across all measurement-noise levels,
demonstrating strong robustness to sensing perturbations.
Moreover, under the moderate kinematic error setting (Fig.~\ref{fig:mea_levels}(a-1,a-2)), \textsf{SDP-Ini} becomes increasingly competitive with the joint calibration method \textsf{Chen}\# as the measurement noise grows.
This is not surprising: when measurement noise is large, the calibration accuracy is progressively limited by the sensing uncertainty in $\boldsymbol{B}_i$,
and thus the marginal benefit of additional (separated) kinematic correction in \textsf{Chen}\# is reduced; in this regime, a strong coordinate solution provided by \textsf{SDP-Ini} can already mitigate the loop inconsistency to a comparable level.

A noteworthy phenomenon is that the \emph{joint} calibration methods (\textsf{Chen}\# and our \textsf{Unified}\#) exhibit larger variability across measurement-noise levels than the \emph{coordinate-only} baselines (\textsf{Wu}, \textsf{Wang}, \textsf{Jiang}) and our \textsf{SDP-Ini}.
This does not imply that joint calibration is less accurate; rather, it reflects the interaction between kinematic and measurement uncertainties.
When the kinematic errors dominate the loop inconsistency, changes in measurement noise have a weaker marginal effect on the final closed-loop deviation, so the apparent noise-induced fluctuation becomes less pronounced.
This can be observed by comparing panels Fig.~\ref{fig:mea_levels}(b) and Fig.~\ref{fig:mea_levels}(a): under the higher kinematic-error setting (b), the performance curves of all \textit{coordinate-only} methods are overall more ``compressed'' across measurement-noise levels,
i.e., the variation induced by measurement noise is relatively smaller than that in (a).

Despite these effects, \textsf{Unified\#} remains the best-performing method in all settings,
maintaining the lowest closed-loop deviations.
This highlights the benefit of jointly optimizing the coordinate transformations and both arms' kinematic parameters under a unified loop-consistency objective,
which mitigates kinematic-induced bias while retaining robustness against increasing measurement noise.

\begin{figure}[!t]
    \centering
    \includegraphics[width=0.475\textwidth]{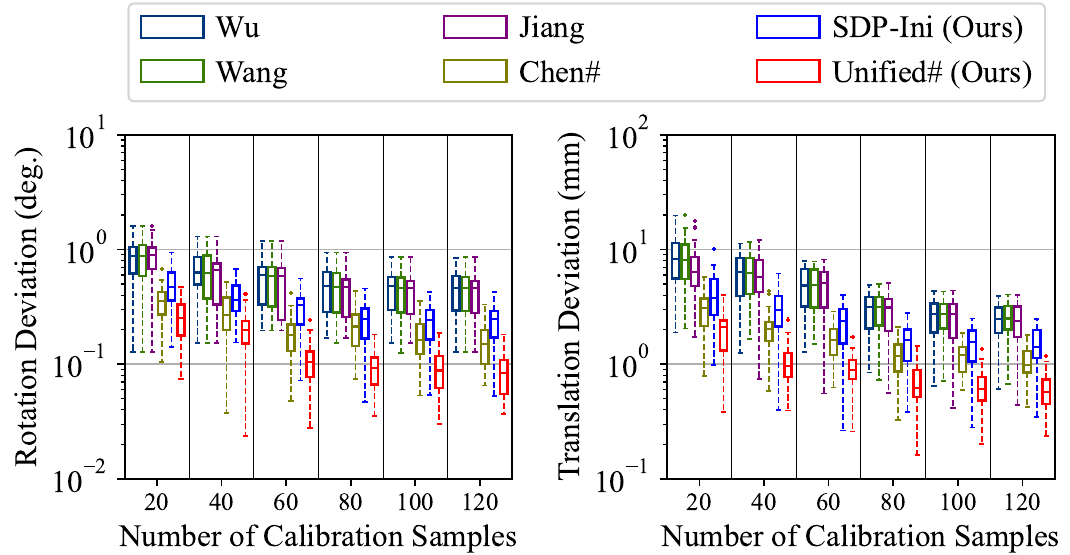}
    \\[-0.3em]
    \makebox[0.24\textwidth]{\footnotesize \quad \quad \quad \ \ (a-1)}
    \makebox[0.24\textwidth]{\footnotesize \quad \quad \ \ (a-2)}
    \\[0.1em]
    \includegraphics[width=0.475\textwidth]{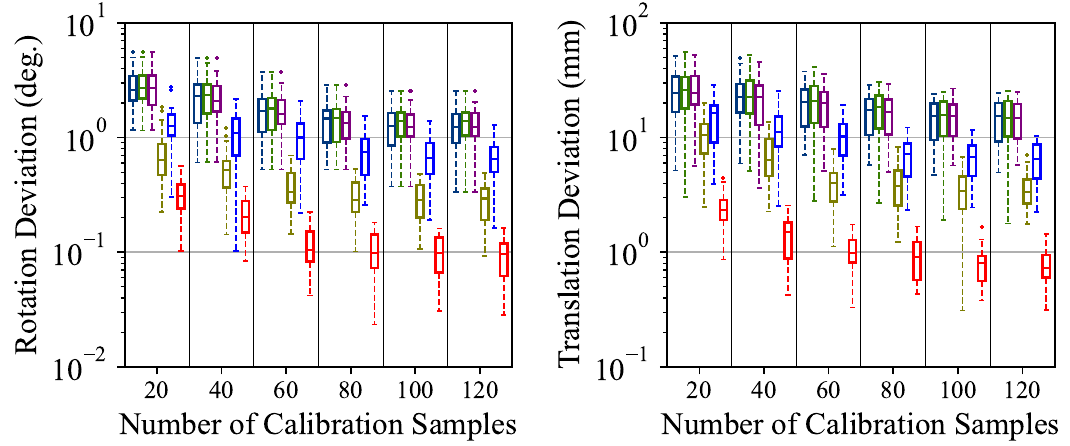}
    \\[-0.4em]
    \makebox[0.24\textwidth]{\footnotesize \quad \quad \quad \ (b-1)}
    \makebox[0.24\textwidth]{\footnotesize \quad \quad \ \ (b-2)}
    \\
    \vspace{-0.25em}
    \caption{Simulated evaluation results on different amounts of calibration samples under (a) medium and (b) high kinematic errors with medium-high measurement noises~(see Section~\ref{subsec:sim_exper}3). Each box plot summarizes the error distribution over 40 distinct test samples.}
    \label{fig:num_samples}
\end{figure}

\vspace{0.5em}
\noindent \textit{ 3) Evaluation on Different Number of Calibration Samples:}
\vspace{0.2em}

We further study how the amount of calibration data affects the accuracy and stability of different methods.
In this experiment, we vary the number of calibration samples as $m\in\{20,40,60,80,100,120\}$, while keeping the test protocol unchanged.
All methods are evaluated under a fixed medium-high measurement-noise setting, and we consider two kinematic-error regimes: medium and high.
For each setup, we use the generated $m$ samples for calibration and evaluate on $40$ additional unseen samples using the closed-loop deviations.
Fig.~\ref{fig:num_samples} summarizes the resulting rotation and translation deviations.

\textbf{Results}.
Fig.~\ref{fig:num_samples}(a-1,a-2) shows that, under medium kinematic errors, increasing $m$ generally reduces both the median deviations and the dispersion for all methods,
indicating that additional calibration data improves numerical stability and mitigates overfitting to individual noisy samples.
Notably, our \textsf{Unified}\# achieves consistently lower deviations even with as few as $m=20$ samples.
This suggests that our unified calibration can effectively exploit limited data.

Under high kinematic errors (see Fig.~\ref{fig:num_samples}(b-1,b-2)), the benefit of more samples becomes even more apparent, and the advantage of our methods is further amplified.
The three coordinate-only baselines remain largely biased because the dominant error source is kinematic and cannot be eliminated by optimizing the coordinate transformations alone; consequently, their medians decrease slowly, and their error distributions stay relatively broad even at $m=120$.
In contrast, our \textsf{SDP-Ini} consistently attains substantially lower loop deviations than the coordinate-only baselines across all sample budgets, demonstrating that the SDP-based initialization is data-efficient and robust even under significant kinematic bias.
The joint calibration baseline, \textsf{Chen}\#, also improves with more samples, but still exhibits larger dispersion and higher deviations than \textsf{Unified\#}, especially in translation.
Overall, \textsf{Unified\#} consistently outperforms all the compared baselines, achieving the lowest deviations with the strongest robustness across all calibration sample sizes.

\vspace{0.5em}
\noindent  \textit{4) Abalation Study on Our Certifiable Initialization: }
\vspace{0.2em}

\begin{figure}[!t]
    \centering
    \includegraphics[width=0.48\textwidth]{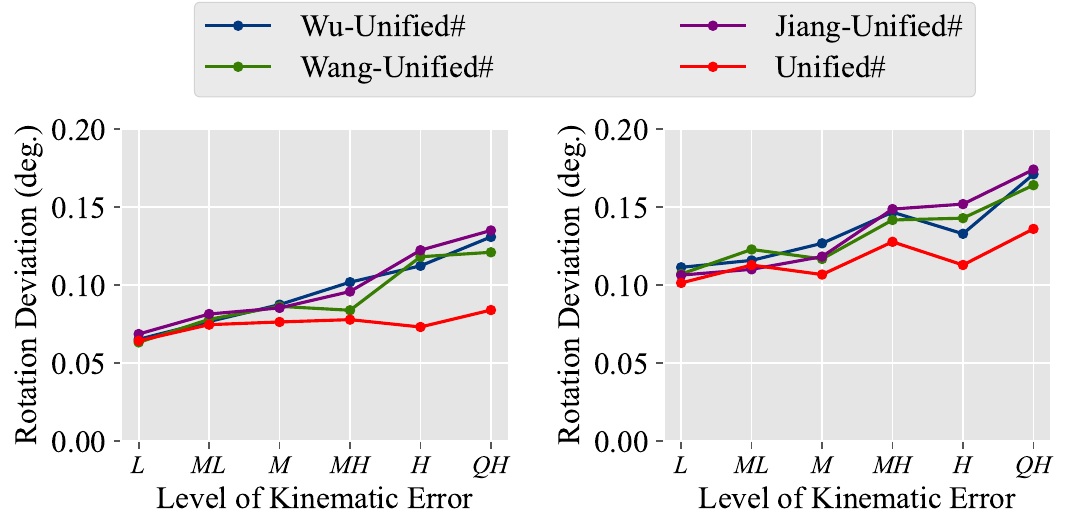}
    \\[-0.4em]
    \makebox[0.24\textwidth]{\footnotesize \quad \quad \quad \ \ (a-1)}
    \makebox[0.24\textwidth]{\footnotesize \quad \quad \ \ (b-1)}
    \\[0.15em]
    \includegraphics[width=0.48\textwidth]{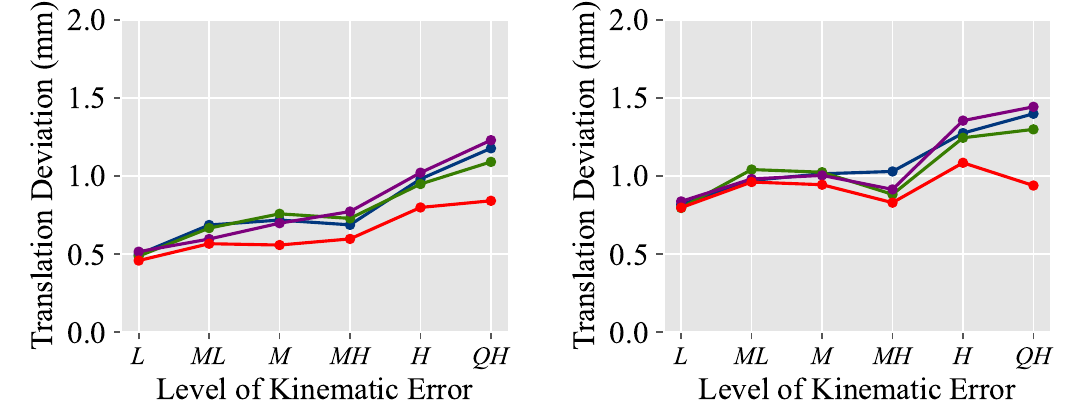}
    \\[-0.4em]
    \makebox[0.24\textwidth]{\footnotesize \quad \quad \quad \ (a-2)}
    \makebox[0.24\textwidth]{\footnotesize \quad \quad \ \ (b-2)}
    \\[0.15em]
    \includegraphics[width=0.48\textwidth]{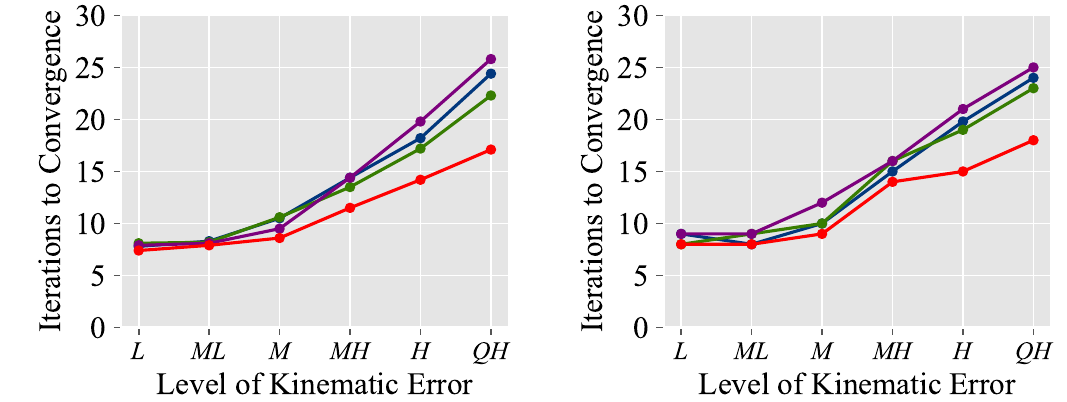}
    \\[-0.3em]
    \makebox[0.24\textwidth]{\footnotesize \quad \quad \quad \ (a-3)}
    \makebox[0.24\textwidth]{\footnotesize \quad \quad \ \ (b-3)}
    \vspace{-1.6em}
    \caption{Simulated ablation study regarding our initialization~(i.e., \textsf{SDP-Ini}) on different levels of kinematic errors under (a) medium and (b) high measurement noises~(see Section~\ref{subsec:sim_exper}4). The curves report the mean deviations over 40 distinct test samples and the iterations to convergence of the unified optimizer with different initializations.}
    \label{fig:ablation_ini}
\end{figure}

To analyze the impact of our certifiable initialization~(i.e., \textsf{SDP-Ini}) in our unified calibration, we construct additional variants by replacing the \textsf{SDP-Ini} module in \textsf{Unified}\# with the three coordinate-only calibration baselines, while keeping our unified joint optimization unchanged.
Specifically, we denote these variants as \textsf{Wu-Unified}\#, \textsf{Wang-Unified}\#, and \textsf{Jiang-Unified}\#. 
We reuse the same simulation data and evaluation protocol as in Section~\ref{subsec:sim_exper}1~(i.e., six kinematic-error levels under medium and high measurement-noise settings).
The results are summarized in Fig.~\ref{fig:ablation_ini}.
Since the performance of these variants is relatively close and the box plots are hard to distinguish, we report the \emph{mean} rotation/translation deviations as line plots for clearer comparison, together with the iterations to convergence of the unified optimizer.

\textbf{Results}.
As shown in Fig.~\ref{fig:ablation_ini}, we can find that all variants achieve remarkably low deviations under low-to-medium kinematic errors, suggesting that our unified optimization can effectively improve a reasonably good coordinate initialization.
Nevertheless, \textsf{Unified}\# with \textsf{SDP-Ini} always achieves the best overall accuracy, and its advantage becomes increasingly evident as the kinematic perturbation grows.
In particular, under high kinematic-error regimes (\textit{H} and \textit{QH}), \textsf{Unified}\# yields much lower mean rotation and translation deviations than \textsf{Wu-Unified}\#, \textsf{Wang-Unified}\#, and \textsf{Jiang-Unified}\# under both medium and high measurement noises.
This highlights that a stronger coordinate initialization remains crucial when the robot-provided poses are heavily biased by kinematic errors: a better initializer places the unified solver closer to a high-quality solution and reduces the risk of drifting toward suboptimal local minima.
The benefit of \textsf{SDP-Ini} is further reflected in the optimization efficiency.
As shown by the iteration curves in Fig.~\ref{fig:ablation_ini}(a-3,b-3), \textsf{Unified}\# typically converges in fewer iterations, with the gap widening at larger kinematic errors; this aligns with the empirical observation that PoE-based kinematic calibration often converges within a small number of iterations given a reasonable initialization~\cite{He-TRO2010}.
Moreover, the faster convergence of \textsf{Unified}\# is consistent with Fig.~\ref{fig:kin_levels}: since \textsf{SDP-Ini} provides a more accurate coordinate estimate than other coordinate-only baselines, such a higher-quality starting point naturally leads to faster convergence.

Overall, this ablation study confirms that while the proposed unified formulation is robust to the choice of coordinate initialization, \textsf{SDP-Ini} is an important component that improves both final accuracy and convergence speed, especially in challenging regimes with substantial kinematic errors.

\subsection{Real-world Experiments}
\label{subsec:real_exper}

\begin{figure}[!t]
    \centering
    \includegraphics[width=0.485\textwidth]{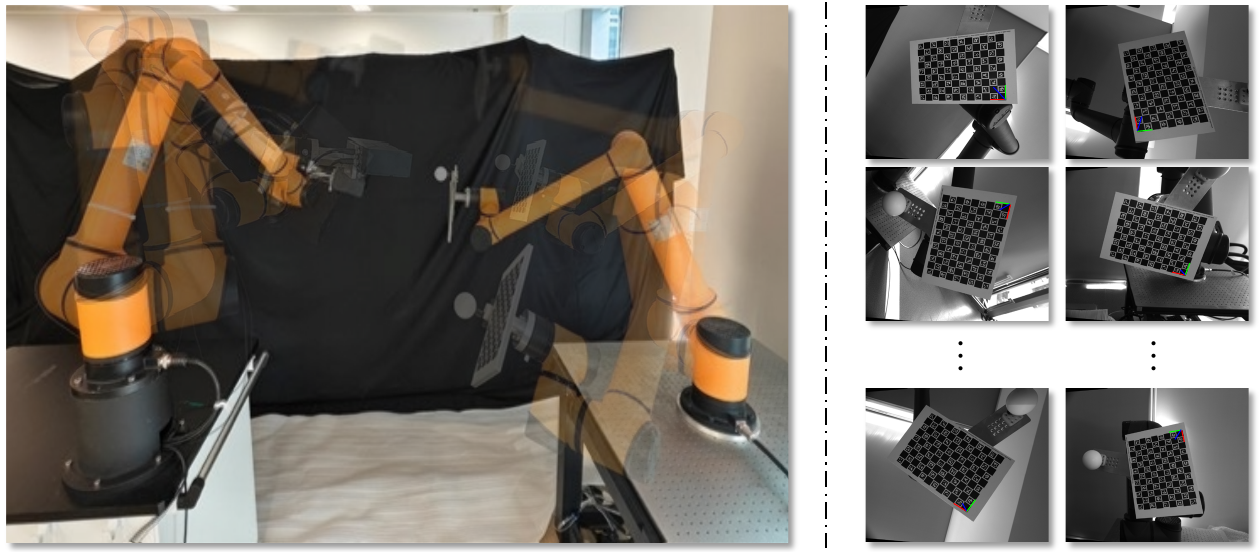}
    \\[0em]
    \makebox[0.24\textwidth]{\footnotesize \quad \quad \quad \ \ (a)}
    \makebox[0.24\textwidth]{\footnotesize \quad \quad  \quad \quad \quad \ (b)}
    \\[-0.4em]
    \caption{Data acquisition in real-world experiments. (a) We synchronically adjust the attitudes of both robot arms, ensuring that distinct measurements are obtained. (b) Example images captured at different arm postures, where the calibration pattern is detected to estimate the tool pose in the sensor frame.} 
    \label{fig:real_data}
\end{figure}

\begin{table*}[!t]
\centering
    \caption{Real-world evaluation results on different amount of calibration samples~(see Section~\ref{subsec:real_exper}1). Mean rotation/translation derivations of each method evaluated on 30 distinct test samples are reported. The best results are \textbf{boldfaced}.}
    \vspace{-0.3em}
    \label{tab:real_num_samples}
    \footnotesize
    \renewcommand{\tabcolsep}{3pt} 
    \renewcommand\arraystretch{1.3}
    \begin{tabular}{c|cc|cc|cc|cc|cc|cc}
        \Xhline{1pt} 
           \multirow{3}{*}{Method} &  \multicolumn{12}{c}{Number of Calibration Samples} \\
           \cline{2-13}
           & \multicolumn{2}{c|}{20} & \multicolumn{2}{c|}{40} & \multicolumn{2}{c|}{60} & \multicolumn{2}{c|}{80} & \multicolumn{2}{c|}{100} & \multicolumn{2}{c}{120} \\
          \cline{2-13}
          & $e_{\boldsymbol{R}}$~(deg.) & $e_{\boldsymbol{t}}$~(mm)  & $e_{\boldsymbol{R}}$~(deg.) & $e_{\boldsymbol{t}}$~(mm)  & $e_{\boldsymbol{R}}$~(deg.) & $e_{\boldsymbol{t}}$~(mm)  & $e_{\boldsymbol{R}}$~(deg.) & $e_{\boldsymbol{t}}$~(mm)  & $e_{\boldsymbol{R}}$~(deg.) & $e_{\boldsymbol{t}}$~(mm)  & $e_{\boldsymbol{R}}$~(deg.) & $e_{\boldsymbol{t}}$~(mm) \\
        \Xhline{0.5pt}
         \textsf{Wu}~\cite{Wu-TRO2016} & 1.3904 & 6.3091 & 1.1626 & 4.2898  & 1.0959 & 4.0361 & 1.0566 & 3.6754 & 0.9863 & 3.9177 & 0.9769 & 3.6729 \\
         \textsf{Wang}~\cite{Wang-TRO2021} & 1.5661 & 6.2315 & 1.2790 & 4.2283 & 1.1347 & 4.0292 & 1.0737 & 3.6786 & 0.9798 & 4.3918 & 0.9568 & 4.1135\\
         \textsf{Jiang}~\cite{Jiang-TRO2023} & 1.4773 & 6.2637 & 1.2465 & 4.3564 & 1.1920 & 4.1233 & 1.1049 & 3.8480 & 1.0134 & 4.3927 & 1.0101 & 4.0489\\
         \textsf{SDP-Ini}~(Ours) & 1.1608 & 4.8550 & 0.8881 & 3.3454 &  0.7879 & 3.1204 & 0.7246 & 2.9885 & 0.7281 & 3.0956 & 0.7191 & 2.8456\\
         \cline{1-13}
         \textsf{Chen}\#~\cite{Chen-TMECH2025} &  0.9852 & 3.8616 & 0.6989 & 2.7360 & 0.5707 & 2.4525 & 0.5638 & 2.7629 & 0.5813 & 2.5704 &  0.5664 & 2.4162 \\
         \textsf{Unified}\#~(Ours) & \textbf{0.5808}  & \textbf{2.1363}  & \textbf{0.5266} & \textbf{1.4364} & \textbf{0.38501} & \textbf{1.2478} & \textbf{0.3734} & \textbf{1.2962}  & \textbf{0.3768} & \textbf{1.3249}  & \textbf{0.3461} & \textbf{1.2436} \\
        \Xhline{1pt}
    \end{tabular}
\end{table*}

In addition to the above simulation analysis, we further conduct real-world experiments to evaluate our dual-arm calibration algorithm in a practical setup.
All experiments are conducted on a dual-arm robot system as shown in Fig.~\ref{fig:real_system}.
The system consists of two AUBO-i10 robot arms. An industrial 3D scanner is mounted to the sensor-side robot arm, while a calibration tool composed of a chessboard and a standard ball is attached to the tool-side robot arm. 
More details of the hardware setup are provided in the implementation details. 
During data collection, we read the joint variables as well as the nominal kinematic parameters of both robot arms from their controllers.
The 3D scanner provides synchronized high-resolution images and high-precision point clouds of the calibration tool, enabling the following evaluation experiments.

\vspace{0.5em}
\noindent  \textit{1) Standard Calibration Evaluation}
\vspace{0.2em}

As shown in Fig.~\ref{fig:real_data}, we synchronously adjust both robot arms within their workspaces to ensure sufficient target visibility, and detect the chessboard to estimate the tool pose in the sensor frame, i.e., the $\boldsymbol{B}_i^\ast$.
We collect a total of $120$ calibration samples under distinct dual-arm posture configurations. These samples are then partitioned into six calibration subsets with sizes $m\in\{20,40,60,80,100,120\}$.
For each subset, we run all compared methods as well as our methods to perform calibration using the corresponding $m$ samples.
For evaluation, we further collect another 30 \emph{unseen} samples under distinct arm postures and use them to compute the closed-loop transformation deviations of each method.

\begin{figure}[!t]
    \centering
    \includegraphics[width=0.475\textwidth]{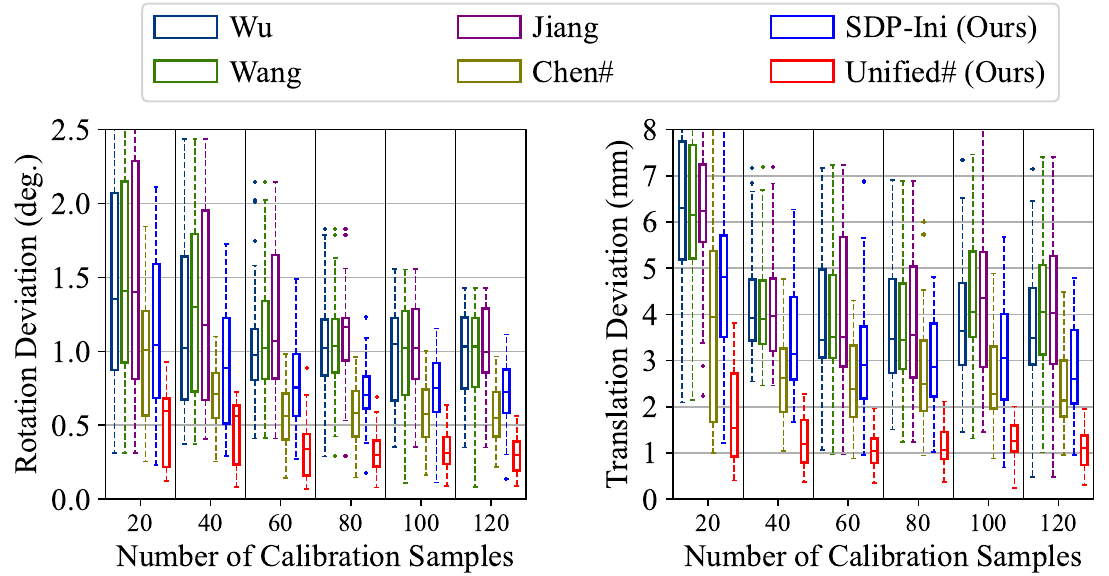}
    \\[-0.3em]
    \makebox[0.24\textwidth]{\footnotesize \quad \quad \quad \ \ (a-1)}
    \makebox[0.24\textwidth]{\footnotesize \quad \quad \ \ (a-2)}
    \\
    \vspace{-0.3em}
    \caption{Real-world evaluation results on different amount of calibration samples~(see Section~\ref{subsec:real_exper}1). Here we use box plots to show the error distributions of each method over 30 distinct test samples.}
    \label{fig:real_num_samples}
\end{figure}

\textbf{Results}.
Fig.~\ref{fig:real_num_samples} reports the real-world closed-loop rotation/translation deviations of different methods under varying numbers of calibration samples. We visualize the error distributions over the 30 test samples using box plots. For a more direct comparison, we further compute the mean rotation/translation deviations and summarize them in Table~\ref{tab:real_num_samples}.

Across all sample numbers, our \textsf{Unified}\# achieves the smallest mean deviations with consistently fewer extreme outliers, demonstrating strong robustness in real-world data.
Notably, even with only $20$ calibration samples, \textsf{Unified}\# already attains low deviations on the majority of the $30$ test samples, with the mean rotation error below 0.6$^\circ$ and mean translation error below 2.2~mm.
This indicates that our unified formulation is highly data-efficient and can reliably calibrate the dual-arm pose chain under limited real-world measurements.
Also, when the calibration set becomes sufficiently large (e.g., $m\ge 60$), \textsf{Unified}\# further reduces the mean rotation deviation to below 0.4$^\circ$, and the median translation deviation to below 1.35~mm, and it is the only method that consistently reaches this accuracy level with a tight error distribution.
In addition, \textsf{SDP-Ini} consistently outperforms the coordinate-only baselines (\textsf{Wu}, \textsf{Wang}, \textsf{Jiang}) under the same sample size, indicating that the proposed SDP relaxation provides a more reliable coordinate estimate. 

As a joint-calibration baseline, \textsf{Chen}\# exhibits a clear improvement over the coordinate-only methods, reflecting the benefit of explicitly compensating kinematic-induced biases. However, its performance gain remains limited, especially compared to our \textsf{SDP-Ini}, which is consistent with its arm-wise error modeling: the residual loop inconsistency cannot be fully and coherently attributed and reduced across the entire pose chain. In contrast, \textsf{Unified}\# further improves upon \textsf{Chen}\# by a remarkable margin (e.g., at $m$ = 100, from 0.5813$^\circ$ / 2.5704~mm to 0.3768$^\circ$ / 1.3249~mm in mean rotation/translation deviations). These results further underscore the importance of our unified optimization that tightly couples coordinate and kinematic terms of both arms, enabling more effective suppression of real-world closed-loop inconsistency.

Fig.~\ref{fig:real_ablation_ini} further studies the role of our certifiable initialization in the unified optimization.
Replacing \textsf{SDP-Ini} with coordinate-only initializers (\textsf{Wu}, \textsf{Wang}, \textsf{Jiang}) still leads to comparable final accuracy when sufficient samples are available, suggesting that the proposed unified refinement is robust once initialized within a reasonable basin.
Nevertheless, \textsf{Unified}\# (with \textsf{SDP-Ini}) consistently yields the lowest mean deviations, especially in the low-data regime ($m=20$ or $40$), where initialization quality plays a more critical role.
Moreover, \textsf{Unified}\# converges in fewer iterations across all sample sizes (Fig.~\ref{fig:real_ablation_ini}(a-3)), confirming that \textsf{SDP-Ini} places the optimizer closer to a high-quality solution and thus reduces the required refinement steps.
Overall, the real-world results corroborate our simulation findings: increasing calibration samples improves accuracy for all methods, while the proposed \textsf{Unified}\# remains the most accurate and data-efficient solution, and \textsf{SDP-Ini} provides both accuracy and convergence benefits for the unified optimization.

\begin{figure}[!t]
    \centering
    \includegraphics[width=0.475\textwidth]{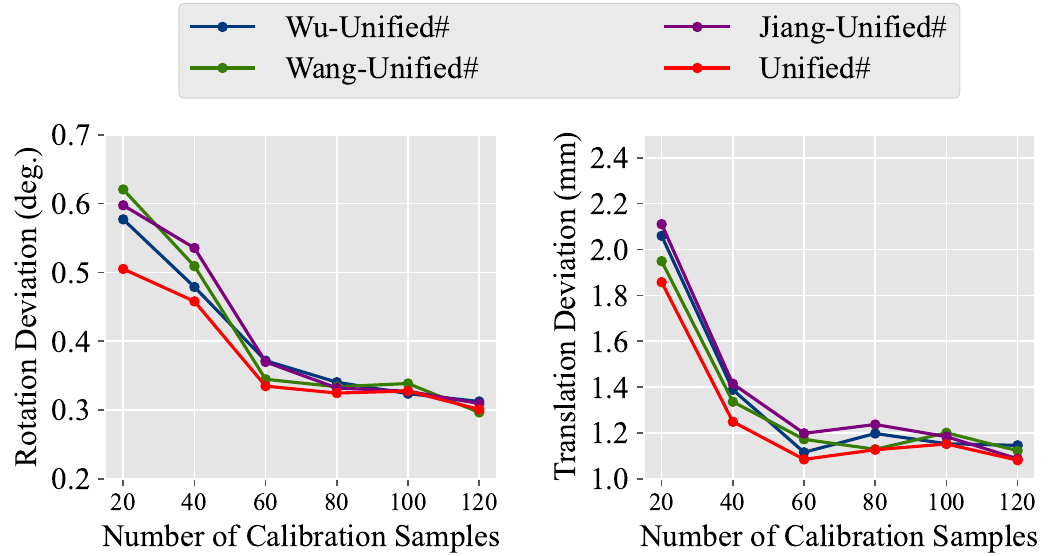}
    \\[-0.4em]
    \makebox[0.24\textwidth]{\footnotesize \quad \quad \quad \ \ (a-1)}
    \makebox[0.24\textwidth]{\footnotesize \quad \quad \ \ (a-2)}
    \\[0.4em]
    \includegraphics[width=0.475\textwidth]{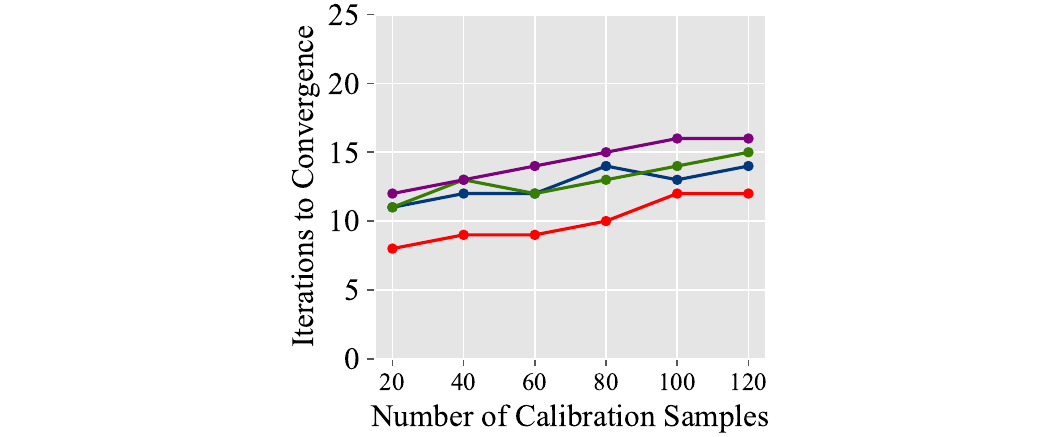}
    \\[-0.4em]
    \makebox[0.24\textwidth]{\footnotesize \quad \quad \quad \ (a-3)}
    \\
    \vspace{-0.5em}
    \caption{Real-world ablation study regarding our initialization~(i.e., \textsf{SDP-Ini}) on different amount of calibration samples~(see Section~\ref{subsec:real_exper}1). The curves report the mean deviations over 30 distinct test samples and the iterations to convergence of the unified optimizer with different initializations.}
    \label{fig:real_ablation_ini}
\end{figure}

\vspace{0.5em}
\noindent \textit{2) Evaluation on Dual-Arm Cooperative Measuring}
\vspace{0.2em}

Beyond the above standard evaluation, we follow~\cite{Wang-TRO2021} and borrow the evaluation idea from VDI/VDE 2634~\cite{Icasio-PE2024} to assess the calibration quality in a representative cooperative-measuring task.
As illustrated in Fig.~\ref{fig:multi_measure}, the dual-arm system can be viewed as a multi-view measuring setup: the sensor-side robot observes the standard ball from multiple viewpoints, while the tool-side robot moves the ball to generate distinct dual-arm postures.
A key property is that the ball is rigidly attached to the tool-side end-effector and thus remains static in the frame $\{E_2\}$.
Therefore, point clouds captured at different postures can be transformed into the same common frame $\{E_2\}$ via the calibrated pose chain for direct alignment.

We collect point clouds of the standard ball under $10$ distinct postures, yielding $\{\mathcal{P}_i\}_{i=1}^{10}$ measured in the camera frame $\{L_1\}$.
Given a calibration solution (either coordinate-only or joint), for each posture $i$ and each point $\boldsymbol{p}\in\mathcal{P}_i$, we can map the point to the common frame $\{E_2\}$ by
\begin{align}
    \boldsymbol{p}' \;=\; \boldsymbol{C}_i^{-1}\,\boldsymbol{Y}^{-1}\,\boldsymbol{A}_i\,\boldsymbol{X}\,\boldsymbol{p},
    \label{eq:ball_transform}
\end{align}
so that all transformed point sets are expressed in $\{E_2\}$ and can be directly aggregated for alignment and subsequent evaluation. In fact, we can access the cooperative measuring accuracy in terms of two aspects: 
\begin{itemize}
    \item \textit{Qualitative registration quality}: we visualize the overlaid point clouds to inspect their multi-view alignment.
    \item \textit{Quantitative consistency metric}: we fit a sphere to each transformed point cloud and obtain the estimated sphere centers $\{\boldsymbol{c}_i\}_{i=1}^{10}$.
    We then compute the minimum enclosing ball (MEB) of these centers and use its radius as the final score (smaller is better), i.e.,
    \begin{align}
        r_{\mathrm{MEB}} \doteq \min_{\boldsymbol{c}}\ \max_{i}\ \|\boldsymbol{c}_i-\boldsymbol{c}\|_2 .
    \end{align}
\end{itemize}

\begin{figure*}[!t]
    \centering
    \makebox[0.16\textwidth]{\footnotesize \quad \quad \textsf{Wu}~\cite{Wu-TRO2016}}
    \makebox[0.16\textwidth]{\footnotesize \quad \ \textsf{Wang}~\cite{Wang-TRO2021}} 
    \makebox[0.16\textwidth]{\footnotesize \quad \textsf{Jiang}~\cite{Jiang-TRO2023}} 
    \makebox[0.16\textwidth]{\footnotesize \textsf{SDP-Ini}~(Ours)}
    \makebox[0.16\textwidth]{\footnotesize \textsf{Chen}\#~\cite{Chen-TMECH2025}}
    \makebox[0.16\textwidth]{\footnotesize \textsf{Unified}\#~(Ours)}     \\
    \vspace{0.3em}
    \includegraphics[width=0.95\textwidth]{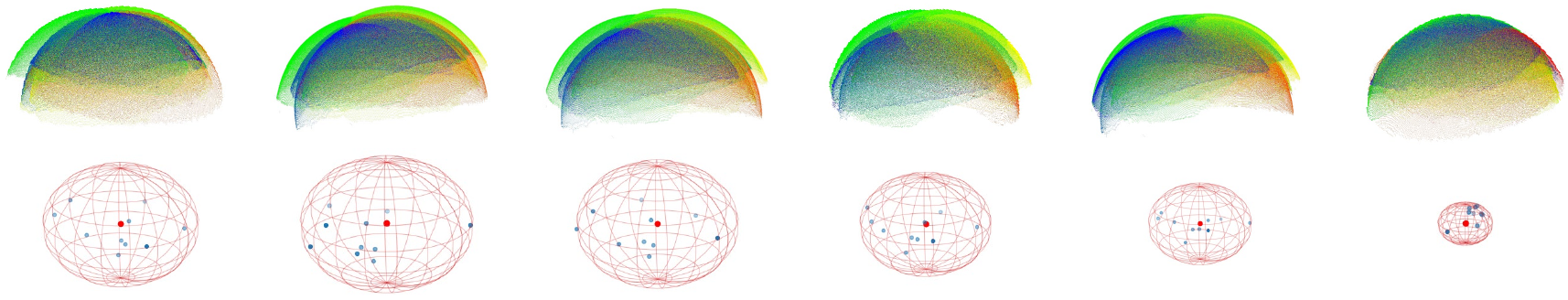}
    \\
    \vspace{-0em}
    \makebox[0.16\textwidth]{\footnotesize \quad \quad $r_{\textnormal{MEB}}$ = 4.5571 mm }
    \makebox[0.16\textwidth]{\footnotesize \quad \ $r_{\textnormal{MEB}}$ = 4.6694 mm} 
    \makebox[0.16\textwidth]{\footnotesize \quad $r_{\textnormal{MEB}}$ = 4.6189 mm} 
    \makebox[0.16\textwidth]{\footnotesize $r_{\textnormal{MEB}}$ = 3.8834 mm}
    \makebox[0.16\textwidth]{\footnotesize $r_{\textnormal{MEB}}$ = 2.9271 mm}
    \makebox[0.16\textwidth]{\footnotesize $r_{\textnormal{MEB}}$ = 1.5105 mm}     \\
    \vspace{-0.2em}
    \caption{The cooperative-measuring evaluation results~(see Section~\ref{subsec:real_exper}2). Below each method are the corresponding aligned point clouds obtained from multiple dual-arm postures, with each transformed point cloud visualized in a different color. The red-grid sphere denotes the minimum enclosing ball (MEB) computed from the centers of the spheres fitted to all transformed point clouds, while the blue points indicate the fitted sphere centers. The number reported below is the MEB radius $r_{\textnormal{MEB}}$, where a smaller value indicates higher cooperative-measuring accuracy.}
    \label{fig:compare_ball}
    \vspace{-0.8em}
\end{figure*}

\begin{figure}[!t]
    \centering
    \includegraphics[width=0.45\textwidth]{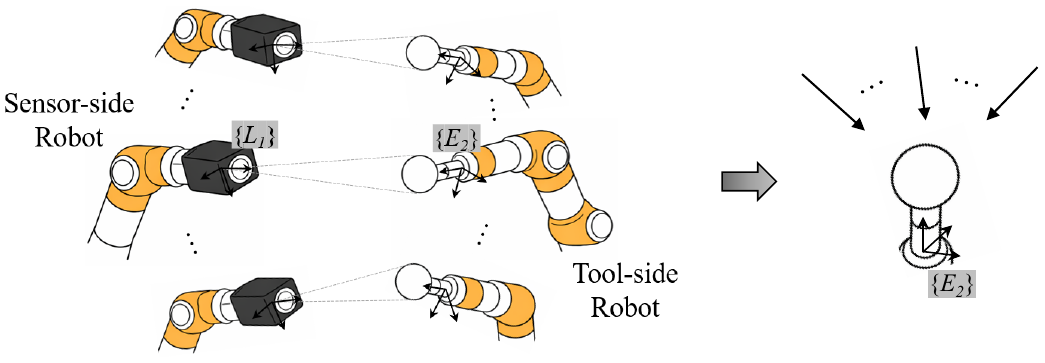}
    \\[-0.35em]
    \caption{The standard ball is rigidly attached to the tool-side end-effector and thus remains static in frame $\{E_2\}$ throughout the experiment.
    For each dual-arm posture, we can transform the ball point cloud measured in $\{L_1\}$ to the common frame $\{E_2\}$ via~\eqref{eq:ball_transform}. As a result, the dual-arm system is now equivalent to a multi-view measuring setup.} 
    \label{fig:multi_measure}
\end{figure}

We use the calibration results obtained in the standard evaluation with $m=100$ calibration samples for each compared method, and apply them to register the above $10$ ball point clouds via~\eqref{eq:ball_transform}. The evaluation results are presented in Fig.~\ref{fig:compare_ball}.

\textbf{Results}.
Fig.~\ref{fig:compare_ball} shows that our \textsf{Unified}\# achieves the best performance in both qualitative and quantitative evaluations. Visually, the overlaid ball point clouds after alignment are the most compact, indicating the smallest cross-view registration drift. Quantitatively, \textsf{Unified}\# yields the \textit{smallest} MEB radius \(r_{\mathrm{MEB}}=1.5105\) mm, providing the strongest evidence that our calibration solution supports highly consistent cooperative measuring across different dual-arm postures.

Among the coordinate-only methods, the three baselines \textsf{Wu}, \textsf{Wang}, and \textsf{Jiang} produce noticeably much less compact overlays and larger MEB radius ($r_{\mathrm{MEB}} > $ 4.5 mm). In contrast, our \textsf{SDP-Ini} significantly improves the coordinate-only regime, achieving a smaller $r_{\mathrm{MEB}}$ = 3.8834 mm and a visibly tighter overlay. This confirms that a stronger coordinate initialization/estimation (with near-global optimality verification) can substantially mitigate loop inconsistency. Nevertheless, \textsf{SDP-Ini} is still fundamentally limited by being coordinate-only: it cannot explicitly correct the kinematic parameters that systematically bias the robot-provided poses, and thus the remaining a few dispersions across multiple views.

The joint calibration baseline \textsf{Chen}\# further decreases $r_{\mathrm{MEB}}$ to 2.9271 mm, demonstrating that incorporating kinematic modeling can indeed benefit cooperative measuring. However, the gain over \textsf{SDP-Ini} is relatively modest, which shows the drawbacks of their separated error modeling.
By contrast, our \textsf{Unified}\# performs a unified optimization where the coordinate and kinematic parameters are coupled in a consolidated error model to minimize the global inconsistency, resulting in a markedly tighter multi-view alignment and a much smaller MEB radius.

In summary, the cooperative-measuring evaluation again highlights that (i) coordinate-only calibration is insufficient when kinematic biases are non-negligible, and (ii) unified error modeling and refinement of the tightly coupled coordinate and kinematic parameters (\textsf{Unified}\#) is crucial to achieve high-precision dual-arm robot cooperation.

\section{Conclusion and Future Work}
\label{sec:conclu}

In this work, we introduce a unified calibration framework that jointly estimates the coordinate transformations and the kinematic parameters of both arms in the vision-based dual-arm robot systems. We construct a consolidated system-level error model grounded in the product-of-exponentials formulation. This modeling avoids arm-wise or residual-level error source separation and therefore mitigates the accumulation of coupled errors along the closed-loop pose chain. With this model, we derive a closed-form analytical Jacobian via Lie derivatives. We analyze the Jacobian rank properties and show that our joint optimization is well-posed under mild excitation conditions.
Furthermore, we develop a certifiably correct initialization for the unknown coordinate transforms via SDP relaxation. The resulting initialization comes with an \textit{a posteriori} near-global optimality certificate, providing a reliable starting point for our unified optimization. 
Extensive simulations and real-world experiments consistently demonstrate the improved calibration accuracy of our approach over related baseline methods under identical visual measurements.

Despite the remarkable accuracy gains achieved by our framework, these results do not imply that the dual-arm robot calibration problem is fully solved. In fact, when conducting the experiments, we observe several remaining open challenges. First, we find that simply increasing the number of samples may not necessarily lead to better performance. For example, in Fig.~\ref{fig:real_ablation_ini}, the accuracy of all methods does not consistently improve when the number of calibration samplers increases from 60 to 100. 
Our inspection suggests that this behavior is largely due to a small fraction of ill-conditioned samples, where the measurements exhibit amplified noise or weak excitation and thus affect the calibration accuracy. A promising research direction is therefore to incorporate an automated data-quality mechanism into the pipeline, such as robust estimation that can detect and remove noisy sample outliers without manual tuning. 
Moreover, another practical challenge is that the calibration data collection for dual-arm robots is quite time-consuming and labor-intensive: collecting 50 dual-arm posture configurations can easily require several hours in real-world settings. This motivates the development of an automated data-acquisition procedure that can actively select informative configurations, potentially via viewpoint/posture planning guided by a coarse initial coordinate calibration. We will explore these directions in future work to further improve both the reliability and the usability of unified dual-arm calibration in real applications.




\bibliographystyle{IEEEtran}
\bibliography{paper}

\clearpage

\appendix
\subsection{Derivation of $\mathcal{J}(\cdot)$ and $\mathfrak{J}(\cdot)$}

Recall that in Section~\ref{subsec:Jaco} of the main manuscript, we leave out the explicit forms of the two basic Jacobian matrices $\mathcal{J}(\cdot)$ and $\mathfrak{J}(\cdot)$ for computing the differential of the exponential map. Here we give a brief derivation for them.

Consider a twist $\hat{\boldsymbol{\xi}} = \begin{bmatrix}
    \boldsymbol{w}_\times & \boldsymbol{\rho} \\
    \boldsymbol{0}_3^{\top} & 0
    \end{bmatrix}$ with $\boldsymbol{w} = [w_1, w_2, w_3]^\top$. 
We begin with the following two definitions:
\begin{align}
    \Omega & \doteq \begin{bmatrix}
        \boldsymbol{w}_\times & \mathbf{0}_{3\times3} \\
        \boldsymbol{v}_\times & \boldsymbol{w}_\times
    \end{bmatrix}, \\
    \theta  & \doteq \sqrt{w_1^2 + w_2^2 + w_3^2}. 
\end{align}
Let $\hat{\boldsymbol{\xi}} = \hat{\boldsymbol{\xi}}(t)$ be a differentiable function with respect to a scalar variable $t$. Then, we can follow \cite{Park-Book1994, Okamura-Robotica1996} to leverage the definite-integral expression and expand it as
\begin{align}\label{eq:J_transtwist}
    [ \delta\exp(\hat{\boldsymbol{\xi}}) \exp(-\hat{\boldsymbol{\xi}}) ]^\vee = &\int_{0}^{1} \exp(\hat{\boldsymbol{\xi}}s) \frac{\mathrm{d} \boldsymbol{\xi} }{\mathrm{d} t} \exp(-\hat{\boldsymbol{\xi}}s) \notag \\
    = & \sum_{k=0}^{\infty} \frac{1}{(k+1)!}\Omega^k \delta \boldsymbol{\xi} \notag \\
    = & \big ( \mathbf{I}_6 + \frac{4 - \theta\sin(\theta) - 4\cos(\theta)}{2\theta^2}\Omega   \notag \\
    & + \frac{4\theta - 5\sin(\theta) + \theta\cos(\theta)}{2\theta^3}\Omega^2 \notag \\
    & + \frac{2 - \theta\sin(\theta) - 2\cos(\theta)}{2\theta^4}\Omega^3 \notag \\
    & + \frac{2\theta - 3\sin(\theta) + \theta\cos(\theta)}{2\theta^5}\Omega^4 \big )\delta \boldsymbol{\xi} \notag \\
    = & \mathcal{J}(\boldsymbol{\xi})\delta\boldsymbol{\xi} 
\end{align}
The last step in \eqref{eq:J_transtwist} gives the explicit form of $\mathcal{J}(\cdot)$ related to a twist with pure transformation. 
Next, we proceed to the Jacobian $\mathfrak{J}(\cdot)$ for a kinematic twist associated with a joint variable $q$. We make the following definition:
\begin{align}
    \Omega' \doteq \begin{bmatrix}
        \boldsymbol{w}_\times q & \mathbf{0}_{3\times3} \\
        \boldsymbol{v}_\times q & \boldsymbol{w}_\times q
    \end{bmatrix}
\end{align}
Similar to the derivation in \eqref{eq:J_transtwist}, we let $\hat{\boldsymbol{\xi}}q = \hat{\boldsymbol{\xi}}(t)q(t)$ be a differentiable function of a scalar variable $t$, then
\begin{align}
    & [ \delta\exp(\hat{\boldsymbol{\xi}} q) \exp(-\hat{\boldsymbol{\xi}} q) ]^\vee \\ 
    = & \sum_{k=0}^{\infty} \frac{1}{(k+1)!}\Omega'^k \delta \boldsymbol{\xi} \notag \\
    = & \big ( \mathbf{I}_6 + \frac{4 - \theta q\sin(\theta q) - 4\cos(\theta q)}{2\theta^2 q}\Omega   \notag \\
    & + \frac{4\theta q - 5\sin(\theta q) + \theta q\cos(\theta q)}{2\theta^3 q}\Omega^2 \notag \\
    & + \frac{2 - \theta q\sin(\theta q) - 2\cos(\theta q)}{2\theta^4 q}\Omega^3 \notag \\
    & + \frac{2\theta q - 3\sin(\theta q) + \theta q\cos(\theta q)}{2\theta^5 q}\Omega^4 \big )\delta \boldsymbol{\xi} \notag \\
    = & \mathfrak{J}(\boldsymbol{\xi}, q)\delta\boldsymbol{\xi}.
\end{align}
With the above expansion, we get the explicit form of $\mathfrak{J}(\cdot)$.

\subsection{Proof of Proposition~\ref{prop:state_linear}: Linear Parameterization of Error Terms in \eqref{eq:coord_ori}}

Here we prove that the error terms $\boldsymbol{f}_i$ and $\boldsymbol{g}_i$ are linear in the aggregated state vector $\boldsymbol{w}$~(ref. Section~\ref{subsec:qcqp}). Recall the definition of $\boldsymbol{w}$:
\begin{align*}
    \boldsymbol{w} \doteq [ &\textnormal{vec}(\boldsymbol{R}_x)^\top, \textnormal{vec}(\boldsymbol{R}_y)^\top, \textnormal{vec}(\boldsymbol{R}_z^\top\otimes\boldsymbol{R}_y)^\top, \notag \\ &  \boldsymbol{t}_x^\top, \boldsymbol{t}_y^\top, \textnormal{vec}(\boldsymbol{t}_z^\top\otimes\boldsymbol{R}_y)^\top, 1]^\top.
\end{align*}
A key identity used throughout our proof is the standard vec–Kronecker relation: Given three matrices $\boldsymbol{M}_1\in \mathbb{R}^{m\times n}$, $\boldsymbol{M}_2 \in \mathbb{R}^{n\times p}$, and $\boldsymbol{M}_3 \in \mathbb{R}^{p\times q}$ with compatible dimensions, there is
\begin{equation}\label{eq:vec_ABC}
\mathrm{vec}(\boldsymbol{M}_1 \boldsymbol{M}_2 \boldsymbol{M}_3)=(\boldsymbol{M}_3^\top\otimes \boldsymbol{M}_1)\mathrm{vec}(\boldsymbol{M}_2).
\end{equation}
Besides, we repeatedly use the identity
\begin{align}\label{eq:vec_Ax}
    \boldsymbol{M}\boldsymbol{t}=(\boldsymbol{t}^\top\otimes \mathbf{I})\mathrm{vec}(\boldsymbol{A})
\end{align}
for any matrix $\boldsymbol{M}$ and vector $\boldsymbol{t}$ with compatible dimensions.

\noindent \textbf{Rotation constraint}.
Recall that the rotation residual vector is defined as $\boldsymbol{f}_i = \mathrm{vec}(\boldsymbol{R}_{a_i}\boldsymbol{R}_{x}\boldsymbol{R}_{b_i}) - \mathrm{vec}(\boldsymbol{R}_{y}\boldsymbol{R}_{c_i}\boldsymbol{R}_{z})$.
Applying \eqref{eq:vec_ABC} to the first term yields
\begin{equation}\label{eq:fi_term1}
\mathrm{vec}(\boldsymbol{R}_{a_i}\boldsymbol{R}_x\boldsymbol{R}_{b_i}) = (\boldsymbol{R}_{b_i}^\top\otimes \boldsymbol{R}_{a_i})\mathrm{vec}(\boldsymbol{R}_x),
\end{equation}
which is linear in $\mathrm{vec}(\boldsymbol{R}_x)$. As to the second term, we can leverage \eqref{eq:vec_ABC}\&\eqref{eq:vec_Ax} and derive as follows:
\begin{align}\label{eq:fi_term2}
\mathrm{vec}(\boldsymbol{R}_y\boldsymbol{R}_{c_i}\boldsymbol{R}_z) &=(\boldsymbol{R}_z^\top\otimes \boldsymbol{R}_y)\mathrm{vec}(\boldsymbol{R}_{c_i}) \notag \\
& = \big(\mathrm{vec}(\boldsymbol{R}_{c_i})^\top\otimes \mathbf{I}_9\big)\mathrm{vec}(\boldsymbol{R}_z^\top\otimes \boldsymbol{R}_y),
\end{align}
which is linear in $\mathrm{vec}(\boldsymbol{R}_z^\top\otimes \boldsymbol{R}_y)$.
Then by combining \eqref{eq:fi_term1} and \eqref{eq:fi_term2}, we can rewrite $\boldsymbol{f}_i$ as a linear function of $\boldsymbol{w}$:
\begin{align}\label{eq:fi_linear}
    \boldsymbol{f}_i = \mathbf{\Omega}_{f_1}\boldsymbol{w},
\end{align}
where one valid choice of $\mathbf{\Omega}_{f_1}$ is
\begin{align}\label{eq:Omega_fi}
    \mathbf{\Omega}_{f_i} = \begin{bmatrix}
        \boldsymbol{R}_{b_i}^\top\otimes \boldsymbol{R}_{a_i} & \mathbf{0}_{9\times9} & - \mathrm{vec}(\boldsymbol{R}_{c_i})^\top\otimes \mathbf{I}_9 & \mathbf{0}_{9\times34}
    \end{bmatrix}
\end{align}

\noindent \textbf{Translation constraint}.
Recall that the translation residual vector is $\boldsymbol{g}_i = \boldsymbol{R}_{a_i}\boldsymbol{R}_{x}\boldsymbol{t}_{b_i} + \boldsymbol{R}_{a_i}\boldsymbol{t}_x + \boldsymbol{t}_{a_i} - \boldsymbol{R}_y\boldsymbol{R}_{c_i}\boldsymbol{t}_z - \boldsymbol{R}_y\boldsymbol{t}_{c_i} - \boldsymbol{t}_y$.  We show that each term in $\boldsymbol{g}_i$ is linear in $\boldsymbol{w}$:

\noindent \emph{(i) Term $\boldsymbol{R}_{a_i}\boldsymbol{R}_{x}\boldsymbol{t}_{b_i}$.}
Relying on \eqref{eq:vec_ABC} and \eqref{eq:vec_Ax}, we have
\begin{align}
\boldsymbol{R}_{a_i}\boldsymbol{R}_{x}\boldsymbol{t}_{b_i}
&=(\boldsymbol{t}_{b_i}^\top\otimes \mathbf{I}_3)\,\mathrm{vec}(\boldsymbol{R}_{a_i}\boldsymbol{R}_x) \notag\\
&=(\boldsymbol{t}_{b_i}^\top\otimes \mathbf{I}_3)\,(\mathbf{I}_3\otimes \boldsymbol{R}_{a_i})\,\mathrm{vec}(\boldsymbol{R}_x),
\end{align}
which is linear in $\mathrm{vec}(\boldsymbol{R}_x)$.

\noindent \emph{(ii) Term $-\boldsymbol{R}_y\boldsymbol{R}_{c_i}\boldsymbol{t}_z$.}
Applying \eqref{eq:vec_Ax} to $\boldsymbol{R}_{c_i}\boldsymbol{t}_z$ gives
\begin{equation}\label{eq:Rci_tz_vec}
\boldsymbol{R}_{c_i}\boldsymbol{t}_z = (\boldsymbol{t}_z^\top\otimes \mathbf{I}_3)\mathrm{vec}(\boldsymbol{R}_{c_i}).
\end{equation}
Left-multiplying \eqref{eq:Rci_tz_vec} by $\boldsymbol{R}_y$ yields
\begin{align}\label{eq:RyRci_tz}
\boldsymbol{R}_y\boldsymbol{R}_{c_i}\boldsymbol{t}_z
&=\boldsymbol{R}_y(\boldsymbol{t}_z^\top\otimes \mathbf{I}_3)\mathrm{vec}(\boldsymbol{R}_{c_i}) \notag \\
&=(\boldsymbol{t}_z^\top\otimes \boldsymbol{R}_y)\mathrm{vec}(\boldsymbol{R}_{c_i}),
\end{align}
where the last equality uses the mixed-product property that can be easily proved. Next, we can again apply \eqref{eq:vec_Ax} and get
\begin{align}
    -\boldsymbol{R}_y\boldsymbol{R}_{c_i}\boldsymbol{t}_z & = - (\boldsymbol{t}_z^\top\otimes \boldsymbol{R}_y)\mathrm{vec}(\boldsymbol{R}_{c_i}) \notag \\
    & = - (\mathrm{vec}(\boldsymbol{R}_{c_i})^\top \otimes \mathbf{I}_3)\mathrm{vec}(\boldsymbol{t}_z^\top\otimes \boldsymbol{R}_y),
\end{align}
which is linear in $\mathrm{vec}(\boldsymbol{t}_z^\top\otimes \boldsymbol{R}_y)$.

\noindent \emph{(iii) Term $-\boldsymbol{R}_y\boldsymbol{t}_{c_i}$.}
Similarly,
\begin{equation}
-\boldsymbol{R}_y\boldsymbol{t}_{c_i}=-(\boldsymbol{t}_{c_i}^\top\otimes \mathbf{I}_3)\mathrm{vec}(\boldsymbol{R}_y),
\end{equation}
which is linear in $\mathrm{vec}(\boldsymbol{R}_y)$.

\noindent \emph{(iv) Remaining terms.}
The terms $\boldsymbol{R}_{a_i}\boldsymbol{t}_x$ and $-\boldsymbol{t}_y$ are trivially linear in $\boldsymbol{t}_x$ and $\boldsymbol{t}_y$, respectively.
The constant offset $\boldsymbol{t}_{a_i}$ is absorbed by the augmented scalar entry $1$ in $\boldsymbol{w}$.

Collecting all terms, $\boldsymbol{g}_i$ admits the  linear form
\begin{equation}\label{eq:gi_linear}
    \boldsymbol{g}_i=\boldsymbol{\Omega}_{g_i}\boldsymbol{w},
\end{equation}
where one valid choice of $\boldsymbol{\Omega}_{g_i}\in\mathbb{R}^{3\times 133}$ is
\begin{align}\label{eq:Omega_gi}
\boldsymbol{\Omega}_{g_i}= [
& (\boldsymbol{t}_{b_i}^\top\otimes \mathbf{I}_3)(\mathbf{I}_3\otimes \boldsymbol{R}_{a_i}) \ \
-(\boldsymbol{t}_{c_i}^\top\otimes \mathbf{I}_3) \ \
\mathbf{0}_{3\times 81} \notag\\
& \boldsymbol{R}_{a_i} \ \
-\mathbf{I}_3 
-\big(\mathrm{vec}(\boldsymbol{R}_{c_i})^\top\otimes \mathbf{I}_3\big) \ \
\boldsymbol{t}_{a_i} ].
\end{align}

\eqref{eq:fi_linear} and \eqref{eq:gi_linear} conclude that both $\boldsymbol{f}_i$ and $\boldsymbol{g}_i$ are linear in $\boldsymbol{w}$, while \eqref{eq:Omega_fi} and \eqref{eq:Omega_gi} provide explicit forms of the corresponding auxiliary matrices $\boldsymbol{\Omega}_{f_i}$ and $\boldsymbol{\Omega}_{g_i}$. The proof ends.

\subsection{Proof of Proposition~\ref{prop:SO3_QC}: Transform SO(3) Constraints to Quadratic Formulas}

Here we prove that the SO(3) constraints in our coordinate-only initialization problem~(see Section~\ref{subsec:qcqp}) can be properly transformed into a finite set of quadratic constraints. 

We begin with some definitions and a basic lemma that can simplify the process of our proof.
We introduce several constant selection matrices to extract SO(3)-related blocks of the state vector $\boldsymbol{w}$:
\begin{align}
    \boldsymbol{S}_x\boldsymbol{w} = \mathrm{vec}&(\boldsymbol{R}_x), \quad \boldsymbol{S}_y\boldsymbol{w} = \mathrm{vec}(\boldsymbol{R}_y),  \\
    \boldsymbol{S}_{zy}\boldsymbol{w} &= \mathrm{vec}(\boldsymbol{R}_z^\top\otimes\boldsymbol{R}_y), \\
    \boldsymbol{S}_{ty}\boldsymbol{w} &= \mathrm{vec}(\boldsymbol{t}_z^\top\otimes\boldsymbol{R}_y),
\end{align}
where the explicit form of each $\boldsymbol{S}$ can be easily derived given the ordering in $\boldsymbol{w}$. Furthermore, there is
\begin{lemma}\label{lemma:qua_bilinear}(Quadratic form for bilinear equalities).
    Given $\boldsymbol{m}_1 = \boldsymbol{M}_1\boldsymbol{w}$ and $\boldsymbol{m}_2 = \boldsymbol{M}_2\boldsymbol{w}$ that are both linear in $\boldsymbol{w}$, any scalar equality of the form $\boldsymbol{m}_1^\top\boldsymbol{Q}\boldsymbol{m}_2 = \rho$ can be rewritten as $\boldsymbol{w}^\top\boldsymbol{H}\boldsymbol{w} = c$ with $\boldsymbol{H}$ being a real symmetric matrix, e.g.,
    \begin{align}
        \boldsymbol{H} = \frac{1}{2}\big( \boldsymbol{M}_1^\top\boldsymbol{Q}\boldsymbol{M}_2 + \boldsymbol{M}_2^\top\boldsymbol{Q}^\top\boldsymbol{M}_1 \big).
    \end{align}
\end{lemma}
\begin{proof}
Substituting $\boldsymbol{m}_1 = \boldsymbol{M}_1\boldsymbol{w}$ and $\boldsymbol{m}_2 = \boldsymbol{M}_2\boldsymbol{w}$ yields
\begin{equation}
\boldsymbol{m}_1^\top\boldsymbol{Q}\boldsymbol{m}_2
=(\boldsymbol{M}_1\boldsymbol{w})^\top \boldsymbol{Q}(\boldsymbol{M}_2\boldsymbol{w})
=\boldsymbol{w}^\top \boldsymbol{M}_1^\top \boldsymbol{Q}\boldsymbol{M}_2\boldsymbol{w}.
\end{equation}
Let $\boldsymbol{K}\doteq \boldsymbol{M}_1^\top \boldsymbol{Q}\boldsymbol{M}_2$. Since $\boldsymbol{w}^\top \boldsymbol{K}\boldsymbol{w}$ is a scalar,
\begin{equation}
\boldsymbol{w}^\top \boldsymbol{K}\boldsymbol{w}
=\big(\boldsymbol{w}^\top \boldsymbol{K}\boldsymbol{w}\big)^\top
=\boldsymbol{w}^\top \boldsymbol{K}^\top \boldsymbol{w}.
\end{equation}
Therefore,
\begin{equation}
\boldsymbol{w}^\top \boldsymbol{K}\boldsymbol{w}
=\boldsymbol{w}^\top \tfrac{1}{2}\big(\boldsymbol{K}+\boldsymbol{K}^\top\big)\boldsymbol{w}.
\end{equation}
Defining $\boldsymbol{H}\doteq \tfrac{1}{2}\big(\boldsymbol{K}+\boldsymbol{K}^\top\big)$, we obtain
$\boldsymbol{H}=\boldsymbol{H}^\top$ and $\boldsymbol{m}_1^\top\boldsymbol{Q}\boldsymbol{m}_2=\boldsymbol{w}^\top \boldsymbol{H}\boldsymbol{w}$,
which proves the claim.
\end{proof}

Now, we are prepared to prove the proposition:

\noindent \textbf{Constraints for $\mathrm{vec}(\boldsymbol{R}_x)$ and $\mathrm{vec}(\boldsymbol{R}_y)$}.
Let $\boldsymbol{R} \in \mathbb{R}^{3\times3}$ denote either $\boldsymbol{R}_x$ or $\boldsymbol{R}_y$, and write $\boldsymbol{R} = \begin{bmatrix}
    \boldsymbol{r}_1, \boldsymbol{r}_2, \boldsymbol{r}_3
\end{bmatrix}$ with $\boldsymbol{r}_l \in \mathbb{R}^3$ being its $l-$th column. Under the column-stacking $\mathrm{vec}(\cdot)$ convention, we define column-extraction matrices $\boldsymbol{E}_l \in \mathbb{R}^{3\times9}$ such that
\begin{align}
    \boldsymbol{r}_l = \boldsymbol{C}_l \mathrm{vec}(\boldsymbol{R}),\quad l=1,2,3.
\end{align}
Then, for $\boldsymbol{R}_x$ and $\boldsymbol{R}_y$, we have 
\begin{align}
    \boldsymbol{r}_{x,l} = \boldsymbol{E}_l\boldsymbol{S}_x\boldsymbol{w},\quad
\boldsymbol{r}_{y,l} = \boldsymbol{E}_l\boldsymbol{S}_y\boldsymbol{w}.
\end{align}
Next, we consider the orthogonal constraints and right-handedness constraints in SO(3) separately.
The orthonormal constraints $\boldsymbol{R}^\top\boldsymbol{R} = \mathbf{I}_3$ are equivalent to
\begin{align}
\boldsymbol{r}_l^\top \boldsymbol{r}_l &= 1,\quad l=1,2,3, \label{eq:Rx_unit} \\
\boldsymbol{r}_l^\top \boldsymbol{r}_k &= 0,\quad 1\le l < k\le 3. \label{eq:Rx_orth}
\end{align}
Given Lemma~\ref{lemma:qua_bilinear}, since each $\boldsymbol{r}_l$ is linear in $\boldsymbol{w}$, each equality in \eqref{eq:Rx_unit}-\eqref{eq:Rx_orth} is a quadratic constraint on $\boldsymbol{w}$. For example, for $\boldsymbol{R}_x$, we have
\begin{align}
    \boldsymbol{r}_{x,l}^\top \boldsymbol{r}_{x,l}
&= (\boldsymbol{E}_l\boldsymbol{S}_x\boldsymbol{w})^\top (\boldsymbol{E}_l\boldsymbol{S}_x\boldsymbol{w}) \notag \\
&= \boldsymbol{w}^\top\big(\boldsymbol{S}_x^\top \boldsymbol{E}_l^\top\boldsymbol{E}_l \boldsymbol{S}_x\big)\boldsymbol{w} = 1, \ l = 1, 2, 3, \\
\boldsymbol{r}_{x,l}^\top \boldsymbol{r}_{x,k}
&= (\boldsymbol{E}_l\boldsymbol{S}_x\boldsymbol{w})^\top (\boldsymbol{E}_k\boldsymbol{S}_x\boldsymbol{w}) \notag \\
&= \boldsymbol{w}^\top\big(\boldsymbol{S}_x^\top \boldsymbol{E}_l^\top\boldsymbol{E}_k \boldsymbol{S}_x\big)\boldsymbol{w} = 0, \ 1 \leq l< k \leq 3,
\end{align}
where the involved matrices can be symmetrized if desired (cf. Lemma~\ref{lemma:qua_bilinear}) to explicitly satisfy the symmetry. The same construction applies to $\boldsymbol{R}_y$ by replacing $\boldsymbol{S}_x$ with $\boldsymbol{S}_y$. 

We further consider the right-handedness constraint
\begin{align}\label{eq:RH_rule}
\boldsymbol{r}_1 \times \boldsymbol{r}_2 = \boldsymbol{r}_3,
\end{align}
which is equivalent to three scalar bilinear equalities. Using the skew operator $_\times$, we can rewrite \eqref{eq:RH_rule} to
\begin{align}
\boldsymbol{r}_{1\times} \boldsymbol{r}_2 - \boldsymbol{r}_3 = \boldsymbol{0}_3.
\end{align}
Each component of $\boldsymbol{r}_{1\times} \boldsymbol{r}_2$ is bilinear in $\boldsymbol{r}_1$ and $\boldsymbol{r}_2$, and all $\boldsymbol{r}_{\cdot}$s are linear in $\boldsymbol{w}$; accordingly to Lemma~\ref{lemma:qua_bilinear}, every scalar equality of the form $\boldsymbol{m}_1^\top\boldsymbol{Q}\boldsymbol{m}_2$ can be rewritten as $\boldsymbol{w}^\top\boldsymbol{H}\boldsymbol{w} = \rho$.
In addition, the $\mathrm{det}(\boldsymbol{R}) = 1$ constraint is a polynomial of degree 3; we follow~\cite{Wise-ArXiv2025} to enforce this constraint using the simpler set constrained by \eqref{eq:RH_rule}. 
Consequently, the SO(3) constraints on $\mathrm{vec}(\boldsymbol{R}_x)$ and $\mathrm{vec}(\boldsymbol{R}_y)$ can be expressed by a finite set of quadratic equalities in $\boldsymbol{w}$.

\noindent \textbf{Constraints for $\mathrm{vec}(\boldsymbol{R}_z^\top\otimes\boldsymbol{R}_y)$}.
For notation simplicity, we define
\begin{align}
    \boldsymbol{K}_{zy} \doteq \boldsymbol{R}_z^\top\otimes \boldsymbol{R}_y \in \mathbb{R}^{9\times 9}, \
\boldsymbol{k}_{zy} \doteq \mathrm{vec}(\boldsymbol{K}_{zy}) = \boldsymbol{S}_{zy}\boldsymbol{w}.
\end{align}
We enforce two types of constraints: (i) orthogonality of $\boldsymbol{K}_{zy}$ itself, and (ii) the local orthogonality introduced by the Kronecker factorization.

Regarding the orthogonality of $\boldsymbol{K}_{zy}$, $\boldsymbol{R}_z, \boldsymbol{R}_y \in $ SO(3) implies $\boldsymbol{K}_{zy} \in$ SO(9) $\subset$ O(9), that is
\begin{align}\label{eq:K_orth}
    \boldsymbol{K}_{zy}^\top \boldsymbol{K}_{zy} = \boldsymbol{I}_9.
\end{align}
Equivalently, each entry of $\boldsymbol{K}_{zy}^\top \boldsymbol{K}_{zy}$ matches the corresponding entry of $\boldsymbol{I}_9$. 
Using the canonical basis vectors $\{\boldsymbol{e}_p\}_{p=1}^{9}$, the $(p,q)$-th entry constraint can be written as
\begin{equation}\label{eq:K_orth_entry_app}
    \boldsymbol{e}_p^\top \boldsymbol{K}_{zy}^\top \boldsymbol{K}_{zy}\boldsymbol{e}_q
    = \boldsymbol{e}_p^\top \boldsymbol{I}_9 \boldsymbol{e}_q
    = \delta_{pq},\ \forall\, p,q\in\{1,\dots,9\}.
\end{equation}
Note that $\boldsymbol{K}_{zy}\boldsymbol{e}_p$ is exactly the $p$-th column of $\boldsymbol{K}_{zy}$, denoted by $\boldsymbol{k}_{zy}^p\in\mathbb{R}^9$.
Thus \eqref{eq:K_orth_entry_app} is equivalent to the column orthogonality constraints
\begin{equation}
    \boldsymbol{k}_{zy}^{p\top} \boldsymbol{k}_{zy}^{q\top} = \delta_{pq},\ \forall\, p,q\in\{1,\dots,9\}.
\end{equation}
which is a quadratic equality in the entries of $\boldsymbol{K}$.
Since $\boldsymbol{k}_{zy}=\boldsymbol{S}_{zy}\boldsymbol{w}$ is linear in $\boldsymbol{w}$, each equality above can be rewritten as some
quadratic equality in $\boldsymbol{w}$ by Lemma~\ref{lemma:qua_bilinear}.

The orthogonality constraint alone only enforces $\boldsymbol{K}_{zy}\in O(9)$, but does not guarantee that
$\boldsymbol{K}_{zy}$ admits the specific Kronecker structure $\boldsymbol{R}_z^\top\otimes \boldsymbol{R}_y$.
We therefore consider the local orthogonality introduced by the Kronecker factorization.

Partition $\boldsymbol{K}_{zy}\in\mathbb{R}^{9\times 9}$ into $3\times 3$ blocks:
\begin{equation}
\boldsymbol{K}_{zy}=
\begin{bmatrix}
\boldsymbol{K}_{11} & \boldsymbol{K}_{12} & \boldsymbol{K}_{13}\\
\boldsymbol{K}_{21} & \boldsymbol{K}_{22} & \boldsymbol{K}_{23}\\
\boldsymbol{K}_{31} & \boldsymbol{K}_{32} & \boldsymbol{K}_{33}
\end{bmatrix},\ \boldsymbol{K}_{pq}\in\mathbb{R}^{3\times 3},\ p,q\in\{1,2,3\}.
\end{equation}
If $\boldsymbol{K}_{zy}=\boldsymbol{R}_z^\top\!\otimes \boldsymbol{R}_y$, then each block must satisfy
\begin{equation}\label{eq:block_relation_app}
\boldsymbol{K}_{pq} = (\boldsymbol{R}_z^\top)_{pq}\,\boldsymbol{R}_y
\quad\Longleftrightarrow\quad
\boldsymbol{M}_{pq}\doteq \boldsymbol{R}_y^\top \boldsymbol{K}_{pq} = (\boldsymbol{R}_z^\top)_{pq}\,\boldsymbol{I}_3.
\end{equation}
Hence, for every $(p,q)$, the matrix $\boldsymbol{M}_{pq}$ must be a scalar multiple of $\boldsymbol{I}_3$.
This property can be enforced by a finite set of scalar equalities requiring (i) all off-diagonal entries to be zero and
(ii) all diagonal entries to be identical:
\begin{align}
(\boldsymbol{M}_{pq})_{ab} &= 0,\ \forall a\neq b,\ \forall p,q\in\{1,2,3\}, \label{eq:M_offdiag_app}\\
(\boldsymbol{M}_{pq})_{11}-(\boldsymbol{M}_{pq})_{22} &= 0,\
(\boldsymbol{M}_{pq})_{22}-(\boldsymbol{M}_{pq})_{33}=0, \notag \\ & \quad \quad  \quad \quad \quad \quad \forall p,q\in\{1,2,3\}. \label{eq:M_diag_app}
\end{align}
Each scalar term in \eqref{eq:M_offdiag_app}--\eqref{eq:M_diag_app} is bilinear in $(\boldsymbol{R}_y,\boldsymbol{K}_{pq})$.
Since $\mathrm{vec}(\boldsymbol{R}_y)=\boldsymbol{S}_y\boldsymbol{w}$ and $\mathrm{vec}(\boldsymbol{K})=\boldsymbol{S}_{zy}\boldsymbol{w}$ are both linear in $\boldsymbol{w}$,
Lemma~\ref{lemma:qua_bilinear} implies that every equality in \eqref{eq:M_offdiag_app}--\eqref{eq:M_diag_app} admits an equivalent quadratic form
$\boldsymbol{w}^\top \boldsymbol{H}\boldsymbol{w}=\rho$ with $\boldsymbol{H}=\boldsymbol{H}^\top$.

\noindent \textbf{Constraints for $\mathrm{vec}(\boldsymbol{t}_z^\top\otimes\boldsymbol{R}_y)$}.
Define
\begin{align}
\boldsymbol{V}_{ty}\doteq \boldsymbol{t}_z^\top\otimes \boldsymbol{R}_y \in \mathbb{R}^{3\times 9}, \
\boldsymbol{v}_{ty}\doteq \mathrm{vec}(\boldsymbol{V}_{ty})=\boldsymbol{S}_{ty}\boldsymbol{w}.
\end{align}
By construction, $\boldsymbol{V}_{ty}$ consists of three $3\times 3$ block-columns:
\begin{equation}
    \boldsymbol{V}_{ty}=\big[\boldsymbol{V}_1\ \boldsymbol{V}_2\ \boldsymbol{V}_3\big],\ \boldsymbol{V}_j = \boldsymbol{t}_{z(j)}\boldsymbol{R}_y,\ \ j\in\{1,2,3\},
\end{equation}
where $\boldsymbol{t}_{z(j)}$ denotes the $j$-th entry of $\boldsymbol{t}_z$.
Left-multiplying by $\boldsymbol{R}_y^\top$ gives the equivalent characterization
\begin{equation}\label{eq:V_consis_app}
    \boldsymbol{N}_j \doteq \boldsymbol{R}_y^\top \boldsymbol{V}_j = \boldsymbol{t}_{z(j)}\boldsymbol{I}_3,\ j=1,2,3,
\end{equation}
where each $\boldsymbol{N}_j$ must be a scalar multiple of the identity.
This constraint can be enforced without introducing $\boldsymbol{t}_z$ explicitly: it suffices to require that
(i) all off-diagonal entries of $\boldsymbol{N}_j$ vanish and (ii) all diagonal entries of $\boldsymbol{N}_j$ coincide:
\begin{align}
    (\boldsymbol{N}_j)_{ab} &= 0,\  \forall a\neq b,\ j=1,2,3, \label{eq:N_offdiag_app}\\
    (\boldsymbol{N}_j)_{11}-(\boldsymbol{N}_j)_{22} &= 0,\
    (\boldsymbol{N}_j)_{22}-(\boldsymbol{N}_j)_{33}=0,\notag \\
    & \quad \quad \quad \quad \quad \quad \quad \ \ j=1,2,3. \label{eq:N_diag_app}
\end{align}
Each scalar equality in \eqref{eq:N_offdiag_app}--\eqref{eq:N_diag_app} is bilinear in $(\boldsymbol{R}_y,\boldsymbol{V}_j)$.
Since $\mathrm{vec}(\boldsymbol{R}_y)=\boldsymbol{S}_y\boldsymbol{w}$ and $\mathrm{vec}(\boldsymbol{V}_{ty})=\boldsymbol{S}_{ty}\boldsymbol{w}$ are linear in $\boldsymbol{w}$,
Lemma~\ref{lemma:qua_bilinear} implies that every equality in \eqref{eq:N_offdiag_app}--\eqref{eq:N_diag_app} admits an equivalent symmetric quadratic form
$\boldsymbol{w}^\top\boldsymbol{H}\boldsymbol{w}=\rho$.

Above all, we have shown that the SO(3) constraints induced by $\boldsymbol{R}_x,\boldsymbol{R}_y,\boldsymbol{R}_z$ in Problem~\eqref{eq:coord_ori} can be encoded by finitely many quadratic equalities in $\boldsymbol{w}$. The proof ends.

\subsection{Proof of Theorem~\ref{theorem:tightness}: Tightness Guarantee of Relaxation}

Let the optimal values of the original QCQP problem in~\eqref{eq:qcqp} and its SDP relaxation in~\eqref{eq:sdp}
be denoted by $p^\star_{\mathrm{QCQP}}$ and $p^\star_{\mathrm{SDP}}$, respectively. Our proof consists of three steps:

\noindent \textbf{Step 1}: \textit{SDP is a relaxation and yields a lower bound.}
Given any feasible $\boldsymbol{w}$ of problem~\eqref{eq:qcqp}, define
$\boldsymbol{W} \doteq \boldsymbol{w}\boldsymbol{w}^\top$.
Then $\boldsymbol{W}\succeq 0$ and for every constraint index $j=1,\dots,h+1$,
\[
\mathrm{tr}(\boldsymbol{H}_j\boldsymbol{W})
=\mathrm{tr}(\boldsymbol{H}_j\boldsymbol{w}\boldsymbol{w}^\top)
=\boldsymbol{w}^\top\boldsymbol{H}_j\boldsymbol{w}
=\rho_j,
\]
hence $\boldsymbol{W}$ is feasible for the SDP problem in~\eqref{eq:sdp}.
Moreover, its objective satisfies
\[
\mathrm{tr}(\boldsymbol{Q}\boldsymbol{W})
=\mathrm{tr}(\boldsymbol{Q}\boldsymbol{w}\boldsymbol{w}^\top)
=\boldsymbol{w}^\top\boldsymbol{Q}\boldsymbol{w}.
\]
Therefore, every QCQP-feasible $\boldsymbol{w}$ induces an SDP-feasible $\boldsymbol{W}$ with the same cost,
hence the SDP feasible set contains the lifted QCQP feasible set and
\[
p^\star_{\mathrm{SDP}} \le p^\star_{\mathrm{QCQP}}.
\]

\noindent \textbf{Step 2}: \textit{Rank-one optimal SDP solution implies tightness.}
Assume now that the SDP admits an optimal solution $\boldsymbol{W}^\star$
with $\mathrm{rank}(\boldsymbol{W}^\star)=1$.
Since $\boldsymbol{W}^\star \succeq 0$ and is rank-one, there exists a vector
$\boldsymbol{w}^\star$ such that
\[
\boldsymbol{W}^\star = \boldsymbol{w}^\star \boldsymbol{w}^{\star\top}.
\]
By feasibility of $\boldsymbol{W}^\star$ for~\eqref{eq:sdp}, for all $j=1,\dots,h+1$,
\[
\rho_j = \mathrm{tr}(\boldsymbol{H}_j\boldsymbol{W}^\star)
= \mathrm{tr}(\boldsymbol{H}_j\boldsymbol{w}^\star\boldsymbol{w}^{\star\top})
= \boldsymbol{w}^{\star\top}\boldsymbol{H}_j\boldsymbol{w}^\star,
\]
so $\boldsymbol{w}^\star$ satisfies all quadratic equalities in~\eqref{eq:qcqp}
(including the homogenization constraint for $j=h+1$).
Hence $\boldsymbol{w}^\star$ is feasible for the QCQP~\eqref{eq:qcqp}.
Furthermore,
\[
\boldsymbol{w}^{\star\top}\boldsymbol{Q}\boldsymbol{w}^\star
=\mathrm{tr}(\boldsymbol{Q}\boldsymbol{w}^\star\boldsymbol{w}^{\star\top})
=\mathrm{tr}(\boldsymbol{Q}\boldsymbol{W}^\star)
= p^\star_{\mathrm{SDP}}.
\]
Since $\boldsymbol{w}^\star$ is QCQP-feasible, we have
\begin{align}
    p^\star_{\mathrm{QCQP}} \le \boldsymbol{w}^{\star\top}\boldsymbol{Q}\boldsymbol{w}^\star
= p^\star_{\mathrm{SDP}}
\end{align}

\noindent \textbf{Step 3}: \textit{Conclude equality of optimal values and global optimality.}
Combining Step 1 and Step 2 yields
\[
p^\star_{\mathrm{SDP}} \le p^\star_{\mathrm{QCQP}}
\le p^\star_{\mathrm{SDP}},
\]
hence $p^\star_{\mathrm{SDP}} = p^\star_{\mathrm{QCQP}}$.
Accordingly, the SDP relaxation is tight, and the $\boldsymbol{w}^\star$ decomposed from the rank-1 $\boldsymbol{W}^\star$ is a globally
optimal solution to the original QCQP problem in~\eqref{eq:qcqp}. The proof ends.

\end{document}